%% file: main.tex
\pdfoutput=1
\documentclass[conference]{IEEEtran}
\IEEEoverridecommandlockouts

\usepackage{cite}
\usepackage{amsmath, amsthm,amssymb,amsfonts}

\usepackage{graphicx}
\usepackage{textcomp}
\usepackage{xcolor}
\usepackage{bm}

\usepackage{hyperref}

\usepackage{subcaption}
\usepackage[ruled,linesnumbered,vlined]{algorithm2e}
\SetAlFnt{\small\sffamily}
\usepackage{multirow}
\usepackage{threeparttable}

\theoremstyle{definition}

\newtheorem{definition}{Definition}
\newtheorem{theorem}{Theorem}

\def\BibTeX{{\rm B\kern-.05em{\sc i\kern-.025em b}\kern-.08em
    T\kern-.1667em\lower.7ex\hbox{E}\kern-.125emX}}
\begin{document}

\title{SPAP: Simultaneous Demand Prediction and Planning for Electric Vehicle Chargers \\in a New City}

\author{
\IEEEauthorblockN{Yizong Wang\IEEEauthorrefmark{1},
Dong Zhao\IEEEauthorrefmark{1},
Yajie Ren\IEEEauthorrefmark{1},
Desheng Zhang\IEEEauthorrefmark{2}, and
Huadong Ma\IEEEauthorrefmark{1}}
\IEEEauthorblockA{\IEEEauthorrefmark{1}Beijing Key Laboratory of Intelligent Telecommunication Software and Multimedia,\\
Beijing University of Posts and Telecommunications, Beijing, China}
\IEEEauthorblockA{\IEEEauthorrefmark{2}Rutgers University, USA}
\IEEEauthorblockA{\{wangyizong, dzhao, renyj\}@bupt.edu.cn, desheng.zhang@cs.rutgers.edu, mhd@bupt.edu.cn}
}

\maketitle
\input{abstract}


\input{1introduction}
\input{2overview}
\input{3prediction}
\input{4planning}

\input{5evaluation}
\input{6discussion}
\input{7relatedwork}
\input{8conclusion}

\bibliographystyle{IEEEtran}
\bibliography{references}
\end{document}

%% file: abstract.tex
\begin{abstract}
	For a new city that is committed to promoting Electric Vehicles (EVs), it is significant to plan the public charging infrastructure where charging demands are high. However, it is difficult to predict charging demands before the actual deployment of EV chargers for lack of operational data, resulting in a deadlock. A direct idea is to leverage the urban transfer learning paradigm to learn the knowledge from a source city, then exploit it to predict charging demands, and meanwhile determine locations and amounts of slow/fast chargers for charging stations in the target city.
	However, the demand prediction and charger planning depend on each other, and it is required to re-train the prediction model to eliminate the negative transfer between cities for each varied charger plan, leading to the unacceptable time complexity.
	To this end, we design a concept and an effective solution of \underline{S}imultaneous Demand \underline{P}rediction \underline{A}nd \underline{P}lanning (\textit{SPAP}): discriminative features are extracted from multi-source data, and fed into an Attention-based Spatial-Temporal City Domain Adaptation Network (\textit{AST-CDAN}) for cross-city demand prediction; a novel Transfer Iterative Optimization (\textit{TIO}) algorithm is designed for charger planning by iteratively utilizing \textit{AST-CDAN} and a charger plan fine-tuning algorithm. Extensive experiments on real-world datasets collected from three cities in China validate the effectiveness and efficiency of \textit{SPAP}. Specially, \textit{SPAP} improves at most 72.5\% revenue compared with the real-world charger deployment.
\end{abstract}

%% file: 1introduction.tex
\section{Introduction}
\label{sec:1introduction}
Electric Vehicles (EVs) are a key technology to reduce air pollution and greenhouse gas emissions. 
The global EV fleet has expanded significantly over the last decade, 
underpinned by supportive policies and technology advances. Only about 17,000 EVs were on the world’s roads in 2010, while the number had swelled to 10 million by 2020; 
meanwhile, the number of publicly accessible chargers increased by 60\% and 46\% in 2019 and 2020 respectively compared with the previous year \cite{EVoutlook20, EVoutlook21}.
Publicly accessible chargers are indispensable where home and workplace charging are unavailable or insufficient to meet needs (e.g., for long-distance travel) \cite{EVoutlook21}.
For a new city that is committed to promoting EVs, the primary task is to build a network of public EV charging stations from scratch for eliminating various concerns (e.g., charger unavailability, range anxiety) of potential EV buyers.
Moreover, given the high investment cost, charging station operators have concerns about the revenue and payback period.
It is reported that the payback period would fall by 2 years for 1.9\% increase in charger utilization ratio \cite{CPreport}.
Accordingly, charging station operators would only want to deploy charging infrastructure where charging demands are high \cite{gopalakrishnan2016demand}.

However, it is challenging to predict charging demands before the actual deployment of EV chargers for lack of operational data in a new city, resulting in a deadlock.
To address this issue, a traditional way is to infer charging demands by leveraging \textit{implicit} information such as parking demands \cite{chen2013locating} and population distribution \cite{xiong2017optimal}.
Unfortunately, such \textit{indirect} method is error-prone particularly when the EV market share is still small \cite{chen2013locating}, as the implicit data have different distributions with charging demands in nature (detailed in Sect. \ref{evaluation:planning}).
Recently, the advanced data acquisition technologies enable us to collect \textit{explicit} data about charging events of EVs, which helps to charger planning \cite{du2018demand, gopalakrishnan2016demand, li2015growing, luo2020d3p}.
However, some popular data sources, such as taxi/bus trajectories \cite{li2015growing, wang2018bcharge} and renting/returning records from electric car-sharing platforms \cite{du2018demand, luo2020d3p}, are only limited to commercial EVs rather than private EVs. For the general charging stations except for those that are used exclusively for commercial EVs, the only available \textit{explicit} data are their charger transaction records \cite{gopalakrishnan2016demand}, whereas it is impossible in a new city.

To address the data scarcity issue, a direct thought is to leverage the emerging urban transfer learning paradigm, which has been successfully applied for various smart city cold-start problems \cite{wang2018smart} such as crowd flow prediction \cite{DBLP:conf/ijcai/WangGMLY19}, human mobility data generation \cite{he2020human}, chain store site recommendation \cite{guo2018citytransfer}, POI recommendations \cite{ding2019learning}, and parking hotspots detection \cite{liu2018will}. More specifically, \textit{considering the similarity and difference of two cities via commonly available datasets, such as map, POI, traffic, etc., can we learn the knowledge on charging demands from a charging station network that is already deployed in other cites, and further exploit it to predict charging demands, and meanwhile determine proper locations and amount of chargers for charging stations in a new city?} However, it is still a non-trivial task, as the existing studies either still need a small amount of \textit{explicit} data in the target city \cite{DBLP:conf/ijcai/WangGMLY19, ding2019learning}, or are very different from our problem settings \cite{ding2019learning, guo2018citytransfer, he2020human,liu2018will}. By contrast, this work does not rely on any \textit{explicit} data in the new city. 

More specifically, we face a great challenge: \textit{the charger demand distribution varies with city-wide charger plans, and in turn, charger planning is dependent on the charging demand prediction, resulting in a deadlock.} To effectively predict charging demands, it is necessary to capture complex spatial-temporal dependencies, affected by various profile factors (numbers of slow/fast chargers in a station and also its nearby stations) and context factors (POIs, road networks, transportation). Furthermore, the data-driven prediction model trained on one city may not be well adapted to another city due to the dissimilar nature (e.g., city scale, development level and strategy) of different cities, which is also known as the domain shift problem, resulting in the negative transfer \cite{5288526}. Even if we have an effective model to predict the charging demands, it is still required to re-train the model to eliminate the negative transfer for each varied charging plan, leading to the unacceptable time complexity.

To this end, we design a novel algorithm named \textit{Transfer Iterative Optimization (TIO)} for simultaneous demand prediction and planning for EV chargers in the target city, by iteratively utilizing an \textit{Attention-based Spatial-Temporal City Domain Adaptation Network (AST-CDAN)} for charger demand prediction and a charger plan fine-tuning algorithm based on the dynamic programming. 
More specifically, we extract discriminative profile and context features from the multi-source data. The \textit{AST-CDAN} is designed for transferring the knowledge on charging demands from the source city to the target city without EV charging stations, which
consists of four components: a \textit{ProfileNet} and a \textit{ContextNet} that learn latent profile and context features from the raw extracted features respectively, a \textit{DemandNet} that predicts the charging demands over different time intervals of one day, and a \textit{DomainNet} that promotes the features from \textit{ProfileNet} and \textit{ContextNet} to deeper domain-invariant representations. The collaboration of the four components effectively address the domain shift problem between cites.
In summary, our main contributions are as follows:
\begin{itemize}
	\item To the best of our knowledge, we are the first to present the concept and solution of \textit{\underline{S}imultaneous Demand \underline{P}rediction \underline{A}nd \underline{P}lanning (SPAP)} for EV chargers in a new city. It is different from the existing work, which conducts charger planning based on charging demands that are assumed to be known in advance or could be independently inferred. We prove the NP-hardness of the problem and the unacceptable time complexity of a straightforward approach (Sect. \ref{sec:2overview}).
	\item We propose a novel \textit{TIO} algorithm for simultaneous demand prediction and planning with a time complexity of $O(|\textit{C}_{TC}|B^2)$, where $B$ is the budget and $|\textit{C}_{TC}|$ is the number of charging stations in the target city (Sect. \ref{sec:4planning}). 
	A novel model \textit{AST-CDAN} is proposed to accurately predict charger demands in the target city, based on the extracted profile and context features. (Sect. \ref{sec:3prediction}).
	\item Extensive experiments on real datasets collected from three cities in China demonstrate the advantages of \textit{SPAP} over competitive baselines. Moreover, \textit{SPAP} improves at most 72.5\% revenue compared with the real-world charger deployment (Sect. \ref{sec:5evaluation}). We have released the code and data for public use\footnote{\href{https://github.com/easysam/SPAP}{https://github.com/easysam/SPAP}}.
\end{itemize}

%% file: 2overview.tex
\section{Overview}\label{sec:2overview}
This section formally defines the \textit{Simultaneous Demand Prediction and Planning} problem, and proves its NP-hardness and the unacceptable time complexity of a straightforward approach. Then we outline our \textit{SPAP} solution framework.

\subsection{Problem Formulation}
\label{subsec:problem}

\begin{definition}[Charging Station]\label{def:charging station}
	A charging station is represented by a tuple $c_i = (l_i, n_i^S, n_i^F, e_i^S, e_i^F, \bm{p}_i^S, \bm{p}_i^F, \bm{y}_i^S, \bm{y}_i^F)$, consisting of the following nine elements:

	\begin{itemize}
		\item $l_i$, the physical location of $c_i$;
		\item $n_i^S$ and $n_i^F$, \# of slow/fast chargers deployed in $c_i$;
		\item $e_i^S$ and $e_i^F$, the unit costs for deploying any one slow/fast charger in $c_i$;
		\item $\bm{p}_i^S=[p_{i1}^S, p_{i2}^S, \ldots, p_{iT}^S]$ and $\bm{p}_i^F=[p_{i1}^F, p_{i2}^F, \ldots, p_{iT}^F]$, the unit service price vectors of each slow/fast charger over $T$ time intervals of one day;
		\item $\bm{y}_{i}^S=[y_{i1}^S, y_{i2}^S, \ldots, y_{iT}^S]$ and $\bm{y}_{i}^F=[y_{i1}^F, y_{i2}^F, \ldots, y_{iT}^F]$, the charging demand vectors of slow/fast chargers over $T$ time intervals of one day, where $y_{it}^S$ and $y_{it}^F$ are defined as the utilization rates of each slow/fast charger during the $t$-th time interval.
	\end{itemize}
\end{definition}

We consider two cities: source and target cities with deployed charging station set $C_{SC}$ and candidate charging station set $C_{TC}$, respectively.
For each deployed charging station $c_i\in C_{SC}$, all of its elements are known;
whereas for each candidate charging station $c_i\in C_{TC}$, only a part of its elements, $(l_i, e_i^S, e_i^F, \bm{p}_i^S, \bm{p}_i^F)$, are known\footnote{The service price has little room for choice due to the operating cost and the business competition, so it is easy to be determined.}.
We require to make a plan for deploying proper numbers of slow/fast chargers in each candidate station of the target city, defined as follows:
\begin{definition}[EV Charger Plan]
	\label{def: charger plan}
	Given a set of candidate stations $C_{TC}$ in the target city, an EV charger plan is a set $\mathcal{N}_{TC}=\{(n_i^S, n_i^F)\mid c_i\in C_{TC}, n_i^S\in \mathbb{N}, n_i^F\in \mathbb{N}\}$. 
	Note that, it is possible that we do not deploy any charger for one candidate station $c_j\in C_{TC}$, i.e., $n_j^S=n_j^F=0$. 
	For convenience, let $\mathcal{N}_{SC}$ denote the charger plan that has been deployed in the source city.
\end{definition}

\label{subsec:definition}
\begin{definition}[\textbf{C}harging \textbf{D}emand \textbf{P}rediction in the	\textbf{T}arget City (CDPT)]
	\label{def:prediction}
	Given the deployed charger plan $\mathcal{N}_{SC}$ in the source city, a specific charger plan $\mathcal{N}_{TC}$ in the target city, the multi-source context data (POIs, transportation, road networks) $\bm{D}_{SC}$ and $\bm{D}_{TC}$ in both the source and target cites, and the historical charging demand data $\bm{Y}_{SC}=\{(\bm{y}_{i}^S, \bm{y}_{i}^F)\mid c_i \in C_{SC}\}$ in the source city, the \textit{CDPT} problem is to learn a function $f$ to predict the charging demands for all the stations in the target city $\widehat{\bm{Y}}_{TC}=\{(\widehat{\bm{y}_{i}^S}, \widehat{\bm{y}_{i}^F})\mid c_i \in C_{TC}\}$:
	\begin{equation}
		\begin{aligned}
		\min_f \quad & error(\widehat{\bm{Y}}_{TC}, \bm{Y}_{TC}) \\
		 \textrm{s.t.} \quad & \widehat{\bm{Y}}_{TC}\!=\!f(\mathcal{N}_{SC}, \mathcal{N}_{TC}, \bm{D}_{SC}, \bm{D}_{TC}, \bm{Y}_{SC})
		\end{aligned}
	\end{equation}
\end{definition}

\begin{definition}[\textbf{C}harger \textbf{P}lanning in the \textbf{T}arget City (CPT)]
	\label{def:planning}
	Given a set of candidate stations $C_{TC}$ in the target city, the deployed charger plan $\mathcal{N}_{SC}$ and the historical charging demand data $\bm{Y}_{SC}$ in the source city, the multi-source data $\bm{D}_{SC}$ and $\bm{D}_{TC}$ in both the source and target cities, a charging demand predictor $f$ and a budget constraint $B$, the CPT problem is to find an EV charger plan $\mathcal{N}_{TC}$ in the target city such that the total revenue $R$ is maximized while the total deployment cost of chargers does not exceed $B$:
	\begin{equation}
		\begin{aligned}
			\max_{\mathcal{N_{TC}}} \quad & 
			R=\sum_{i=1}^{|C_{TC}|}\sum_{t=1}^{T}{(\widehat{y_{it}^S}  \cdot p_{it}^S \cdot n_i^S + \widehat{y_{it}^F}  \cdot p_{it}^F  \cdot  n_i^F)}\\
			\textrm{s.t.} \quad 
			& \sum_{i=1}^{|C_{TC}|}e_i^S \cdot n_i^S + e_i^F \cdot n_i^F \leq B\\
			& \widehat{\bm{Y}}_{TC}=f(\mathcal{N}_{SC}, \mathcal{N}_{TC}, \bm{D}_{SC}, \bm{D}_{TC}, \bm{Y}_{SC})\\
			& 0 \leq n_i^S \leq u^S \quad\mathrm{and}\quad 0 \leq n_i^F \leq u^F
		\end{aligned}
	\end{equation}
	Note that the charger numbers in each station are constrained by $u^S$ and $u^F$ to avoid unrealistic charger allocation.
\end{definition}

\subsection{Problem Complexity Analysis}
\label{subsec:chllanges}
In this subsection, we prove the NP-hardness of the CPT problem and analyze the time complexity of a straightforward approach.
\begin{theorem}
	The CPT problem is NP-hard.
	\label{theorem:cpt:nphard}
\end{theorem}
\begin{proof}
	We prove the NP-hardness of the \textit{CPT} problem by reducing the unbounded knapsack (\textit{UKP}) problem \cite{andonov2000unbounded} to a special case of the \textit{CPT} problem where $\bm{Y}_{TC}$ is $\mathcal{N}_{TC}$-independent.
	
	The \textit{UKP} problem is illustrated as follows:
	given a knapsack of capacity $c>0$ and $n$ types of items, where each item of type $i$ has value $v_i>0$ and weight $w_i>0$, the objective is to find the number $x_i>0$ of each type of item such that the total value $\sum_{i=1}^n x_iv_i$ is maximized while the total weight does not exceed the capacity, $\sum_{i=1}^n x_iw_i\leq c$.

	If $\bm{Y}_{TC}$ is $\mathcal{N}_{TC}$-independent, then the \textit{CPT} problem is illustrated as a special case: given a budget $B$ and a set of charging stations $\textit{C}_{TC}$, where each station $c_i \in \textit{C}_{TC}$ is represented as a tupel $(l_i$, $n_i^S$, $n_i^F$, $e_i^S$, $e_i^F$, $\bm{p}_i^S$, $\bm{p}_i^F$, $\bm{y}_i^S$, $\bm{y}_i^F)$ (Def. \ref{def:charging station}), the objective is to determine a charger plan $\mathcal{N}_{TC}=\{(n_i^S, n_i^F) \mid i=1, \cdots, |\textit{C}_{TC}|\}$ such that the total revenue is maximized while the total cost of deploying chargers does not exceed the budget:

	\begin{equation}
		\begin{aligned}
			\max_{\mathcal{N}_{TC}} \quad & 
			\sum_{i=1}^{|\textit{C}_{TC}|}\sum_{t=1}^{T}{(y_{it}^S  \cdot p_{it}^S \cdot n_i^S + y_{it}^F  \cdot p_{it}^F  \cdot  n_i^F)}\\
			\textrm{s.t.} \quad 
			& \sum_{i=1}^{|\textit{C}_{TC}|}e_i^S \cdot n_i^S + e_i^F \cdot n_i^F \leq B
		\end{aligned}
	\end{equation} 
	
	Given $W=\{w_i\mid i=1,\cdots,n\}, V=\{v_i\mid i=1,\cdots,n\}$ and $X=\{x_i\mid i=1, \cdots, n\}$, we map an instance of the \textit{UKP} problem, $I = (W, V, X, n, c)$, with an \textbf{even} $n$, to the
	instance of the \textit{CPT} problem where $\textbf{Y}_{TC}$ is $\mathcal{N}_{TC}$-independent, denoted by $I' = (C_{TC}, B)$, as follows: $c$ is mapped to $B$; $n/2$ is mapped to $|\textit{C}_{TC}|$; for any $i=1, 2, \cdots, n/2$, $w_{2i-1}$ is mapped to the slow charger cost $e_i^S$ of $c_i \in C_{TC}$, $w_{2i}$ is mapped to the fast charger cost $e_i^F$ of $c_i \in C_{TC}$, $v_{2i-1}$ is mapped to the daily revenue $\sum_{t=1}^{T}{(y_{it}^S  \cdot p_{it}^S)}$ of $c_i \in C_{TC}$, and $v_{2i}$ is mapped to the daily revenue $\sum_{t=1}^{T}{(y_{it}^F  \cdot p_{it}^F)}$ of $c_i \in C_{TC}$.

	On the one hand, if there is a solution for the instance $I$, $X = (x_1, x_2, \cdots, x_n)$,
	then $\{(n_i^S, n_i^F)\mid n_i^S = x_{2i-1}, n_i^F = x_{2i}, i = 1, \cdots$, $|\textit{C}_{TC}|\}$ is a solution for the instance $I'$.

	On the other hand, if there is a solution for the instance $I'$, $\{(n_i^S, n_i^F)\mid i = 1, \cdots, |\textit{C}_{TC}|\}$, then the numbers
	\begin{equation}
		x_i=\left\{
		\begin{aligned}
			& n_i^S, \text{ if } i \in \{1, 3, \cdots, n-1\} \\
			& n_i^F, \text{ if } i \in \{2, 4, \cdots, n\}
		\end{aligned}
		\right.
	\end{equation}
	are a solution for the instance $I$.
	
	Thus, as long as there is a solution for the \textit{UKP} problem, there is a solution for the special case of the \textit{CPT} problem where $\bm{Y}_{TC}$ is $\mathcal{N}_{TC}$-independent, and vice versa. 
	Then the \textit{UKP} problem can be reduced to the simplified \textit{CPT} problem. Since the \textit{UKP} problem is NP-\textit{hard} \cite{andonov2000unbounded}, the general \textit{CPT} problem is NP-\textit{hard}.
\end{proof}

Note that if $\widehat{\bm{Y}}_{TC}$ is $\mathcal{N}_{TC}$-independent, then the CPT problem can be reduced to an unbounded Knapsack problem \cite{andonov2000unbounded}, which can be solved by dynamic programming or approximation algorithms.
Indeed, the existing studies on charger planning generally determine charging demands in advance by estimating from historical data \cite{du2018demand} or leveraging a plan-independent demand prediction method \cite{gopalakrishnan2016demand}, and thus the charger planning problem can be transformed into the well-known Knapsack and Set-Cover problems or their variants.
However, these studies do not apply to a new city.
Now let us return to our problem setting where $\widehat{\bm{Y}}_{TC}$ is $\mathcal{N}_{TC}$-dependent. 
In essence, the charging demands $\widehat{\bm{Y}}_{TC}$ are determined by a non-linear function of $\mathcal{N}_{TC}$, which requires to be trained with a deep learning model (see Sect. \ref{sec:3prediction}).
Thus, the existing solutions, whether dynamic programming or other approximation algorithms, are not directly applicable any more.
Alternatively, a straightforward approach could be used, which finds the optimal solution from all the possible charger plans ($\mathcal{N}_{TC}$) by the brute-force search. However, it has an unacceptable time complexity as follows.

\begin{theorem}
	If $e_i^S = e_i^F = 1, \forall c_i\in C_{TC}$, then the CPT problem has $\binom{B + 2|\textit{C}_{TC}|-1}{2|\textit{C}_{TC}|-1}$ possible charger plan solutions by the brute-force search.
	\label{theorem:cpt:solution}
\end{theorem}
\begin{proof}
	If $e_i^S = e_i^F = 1, \forall c_i\in C_{TC}$, then the budget $B$ is equal to the total number of chargers that we can deploy. Under this case, the number of the possible charger plan solutions for the \textit{CPT} problem can be proved in two steps.
	
	First, we change the constraints $n_i^S \geq 0$ and $n_i^F \geq 0$ to $n_i^S \geq 1$ and $n_i^F \geq 1$. The number of possible charger plans will be $\binom{B-1}{2|\textit{C}_{TC}|-1}$ by splitting $B$ to the $2|\textit{C}_{TC}|$ parts (for 2 charger types in $|\textit{C}_{TC}|$ stations) according to the stars and bars method in the context of combinatorial mathematics.
	
	Second, we add a ``virtual'' charger to each charger type of each station in advance, and accordingly the budget is increased by $2|\textit{C}_{TC}|$. 
	Similarly, the number of possible charger plans is $\binom{B + 2|\textit{C}_{TC}|-1}{2|\textit{C}_{TC}|-1}$.
	Note that the ``virtual'' chargers are placeholders to satisfy the changed constraints, which do not actually exist.
	After removing the ``virtual'' charger in each charger type of each station, the budget is still $B$, but the original constraints $n_i^S \geq 0$ and $n_i^F \geq 0$ can be satisfied. As a result, the number of the possible charger plan solutions for the \textit{CPT} problem is $\binom{B + 2|\textit{C}_{TC}|-1}{2|\textit{C}_{TC}|-1}$.
\end{proof}

Now we consider a small-scale problem setting with $|C_{TC}|=5$ and $B=100$ for example: given the time of 1 millisecond for demand prediction with a candidate charger plan, the total time required to traverse through all the $\binom{109}{9}$ plans will reach 137 years!
Not only that, for each changed charger plan, it requires to re-train the demand prediction model; given the time of 1 hour for training a model with a candidate charger plan, the total time required to train all the possible models will grow to $4.87\times10^8$ years!
Thus, it is necessary to design an effective solution that is able to greatly reduce the required number of trainings and predictions.

\subsection{Solution Framework of \textit{SPAP}}
\label{subsec:framework}

Figure \ref{fig:framework} gives the framework of \textit{SPAP}, consisting of two components: \textit{charger demand prediction} and \textit{charger planning}, which coordinate to make the charger plan with the highest revenue.

\noindent\textbf{Charger Demand Prediction.} This component addresses the CDPT problem with the following two modules:
\begin{itemize}
	\item{\textit{Feature Extraction}}. It extracts discriminative \textit{profile} and \textit{context features} for charging stations from both source and target cites (Sect. \ref{subsec:feature}).
	\item{\textit{Attention-based Spatial-Temporal City Domain Adaptation Network (AST-CDAN)}}. It leverages the features from both source and target cities and the demand data from source city to predict the charging demands in the target city (Sect. \ref{subsec:AST-CDAN}).

\end{itemize}

\noindent\textbf{Charger Planning.} This component addresses the CPT problem with the following two modules:
\begin{itemize}
	\item{\textit{Transfer Iterative Optimization}}. To greatly reduce the required number of trainings and predictions, the \textit{TIO} algorithm is designed to iteratively utilize the \textit{AST-CDAN} for charger demand prediction and the \textit{Charger Plan Fine-tuning} module to update the charger plan (Sect. \ref{subsec:tio}).
	
	\item{\textit{Charger Plan Fine-tuning}}. It fine-tunes the current charger plan to maximize the total revenue constrained by the budget using a dynamic programming algorithm (Sect. \ref{subsec:optimization}). 
\end{itemize}

\begin{figure}
	\includegraphics[width=\linewidth]{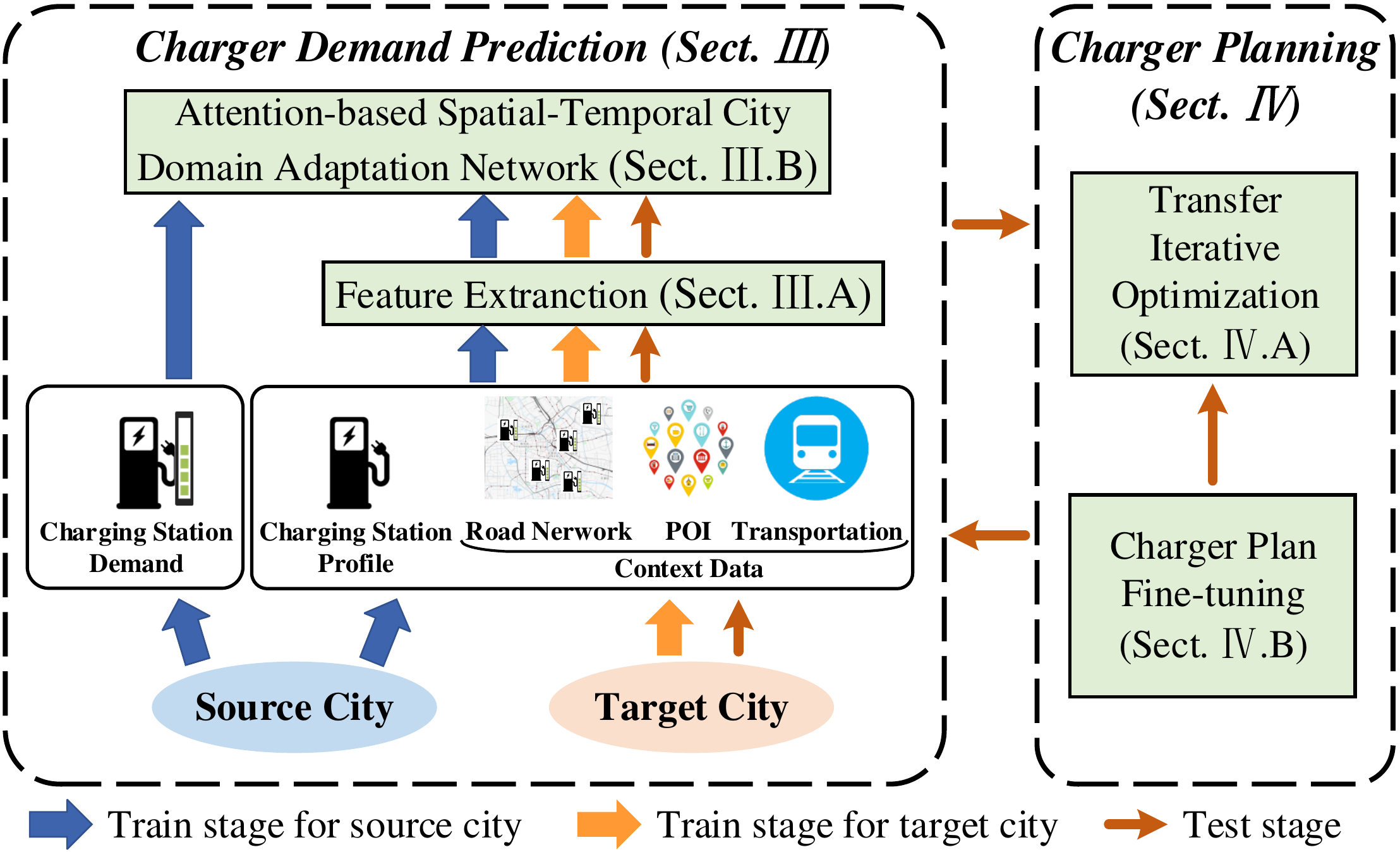}

	\caption{The solution framework of \textit{SPAP}}
	\label{fig:framework}\vspace{-10pt}
\end{figure}

%% file: 3prediction.tex
\section{Charger Demand Prediction}
\label{sec:3prediction}

\subsection{Feature Extraction}
\label{subsec:feature}
To predict the charging demands, we extract the context and profile features of each charging station, and then analyze their correlations and also the feature domain shift between two cities.
\subsubsection{Context Features}
Intuitively, the number and diversity of POIs reflect the prosperity, and the surrounding road network and transportation conditions of a charging station reflect its convenience, all of which have influences on charging demands.
Thus, we extract useful context features in the surrounding region (within radius $r$) of each charging station:
\begin{itemize}
	\item{\textit{POI Features.}}
	We classify POIs into 8 categories: company, school, hotel, fast food, spot, community, hospital and life service. 
	Then, a 17-D POI feature vector is extracted, including \textit{fraction of POIs in each category}, \textit{number of POIs in each category} and \textit{POI entropy}. 

	\item\textit{Road Network Features.} 
	They include the \textit{average street length}, \textit{intersection density}, \textit{street density}, and \textit{normalized degree centrality of intersections}\footnote{The normalized degree centrality is defined as the proportion of links incident upon a node (i.e., the proportion of intersections connected to the given intersection).}, obtained from the nearby streets.

	\item{\textit{Transportation Features.}}
	They include the \textit{number of subway stations}, \textit{number of bus stops} and \textit{number of parking lots}.
\end{itemize}
The above features are concatenated as a single vector and fed into the prediction model.

\subsubsection{Profile Features}
Intuitively, the charging demand of a station $c_i$ is affected not only by the amount and type of chargers deployed in the station itself, but also by the nearby stations $NS(c_i)$ in its surrounding region (within radius $r$).
Thus we extract the profile feature vector as $[|NS(c_i)|, \sum_{c_j\in NS(c_i)}(n_j^S+n_j^F), n_i^S, n_i^F, n_i^S+n_i^F]$.

\begin{table}[htbp]
	\centering
	\caption{Top 8 features with highest Pearson coefficients}
	\label{table:pearson}
	\begin{tabular}{ll||ll}
		\hline
		Feature & Pearson & Feature & Pearson \\ \hline
		\# of fast chargers & 0.6678 & \# of spot POIs & 0.3087 \\
		\# of community POIs & 0.4244 & \# of slow chargers & -0.2945 \\
		\# of parking lots & 0.4010 & \# of school POIs & -0.2927 \\
		\# of all chargers & 0.4609 & Street density & 0.2632 \\ \hline
	\end{tabular}
\end{table}

\subsubsection{Feature Correlation Analysis} Table \ref{table:pearson} lists 8 features that are most correlated with the charging demands. It shows that the absolute Pearson coefficients are all above 0.26, indicating the effectiveness of the selected features for charging demand prediction.

\begin{figure}[t]
	\begin{subfigure}{0.5\linewidth}
		\includegraphics[width=\linewidth]{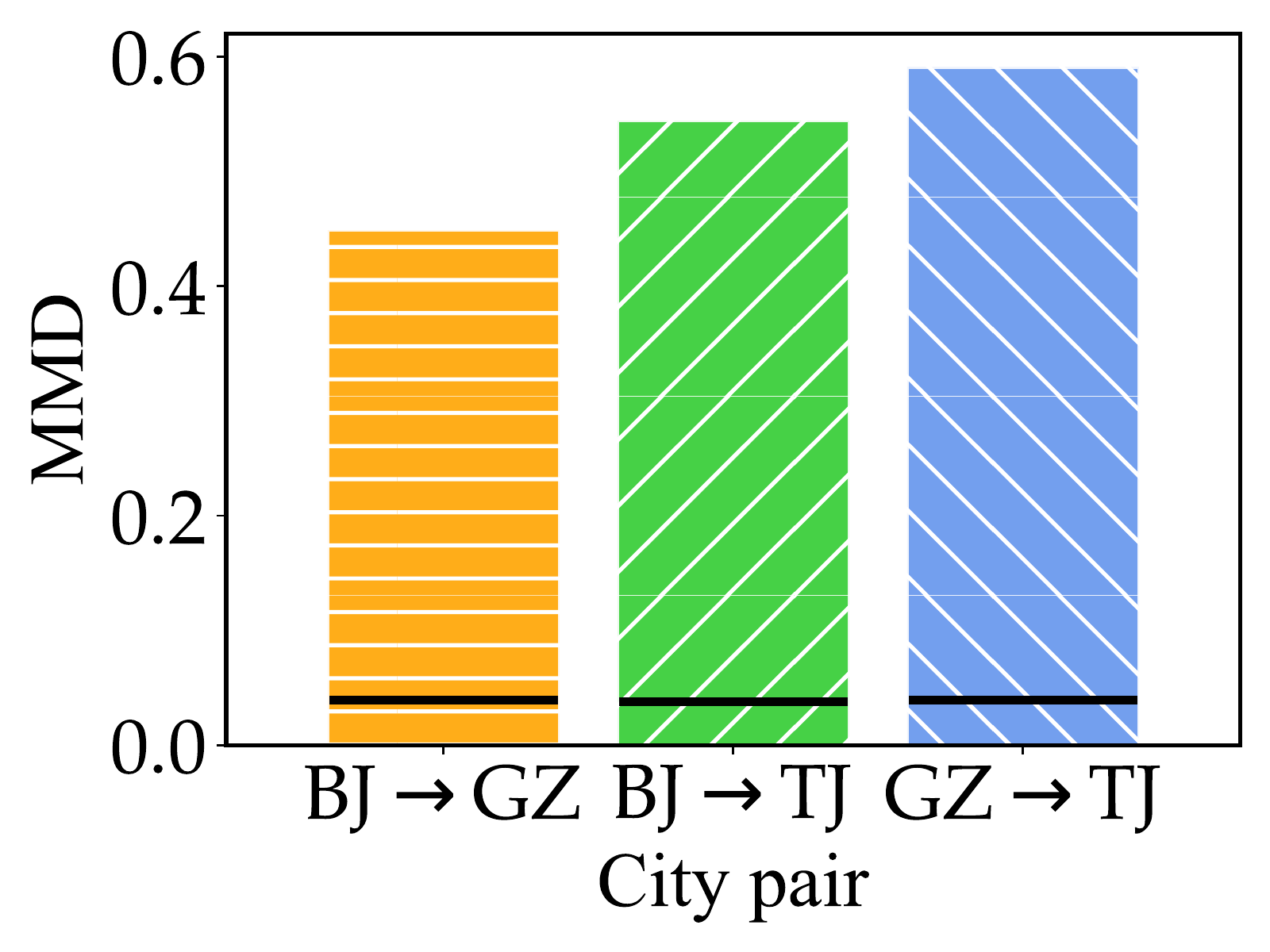}
		\subcaption{MMD results}
		\label{fig_mmd}
	\end{subfigure}%
	\begin{subfigure}{0.5\linewidth}
		\includegraphics[width=\linewidth]{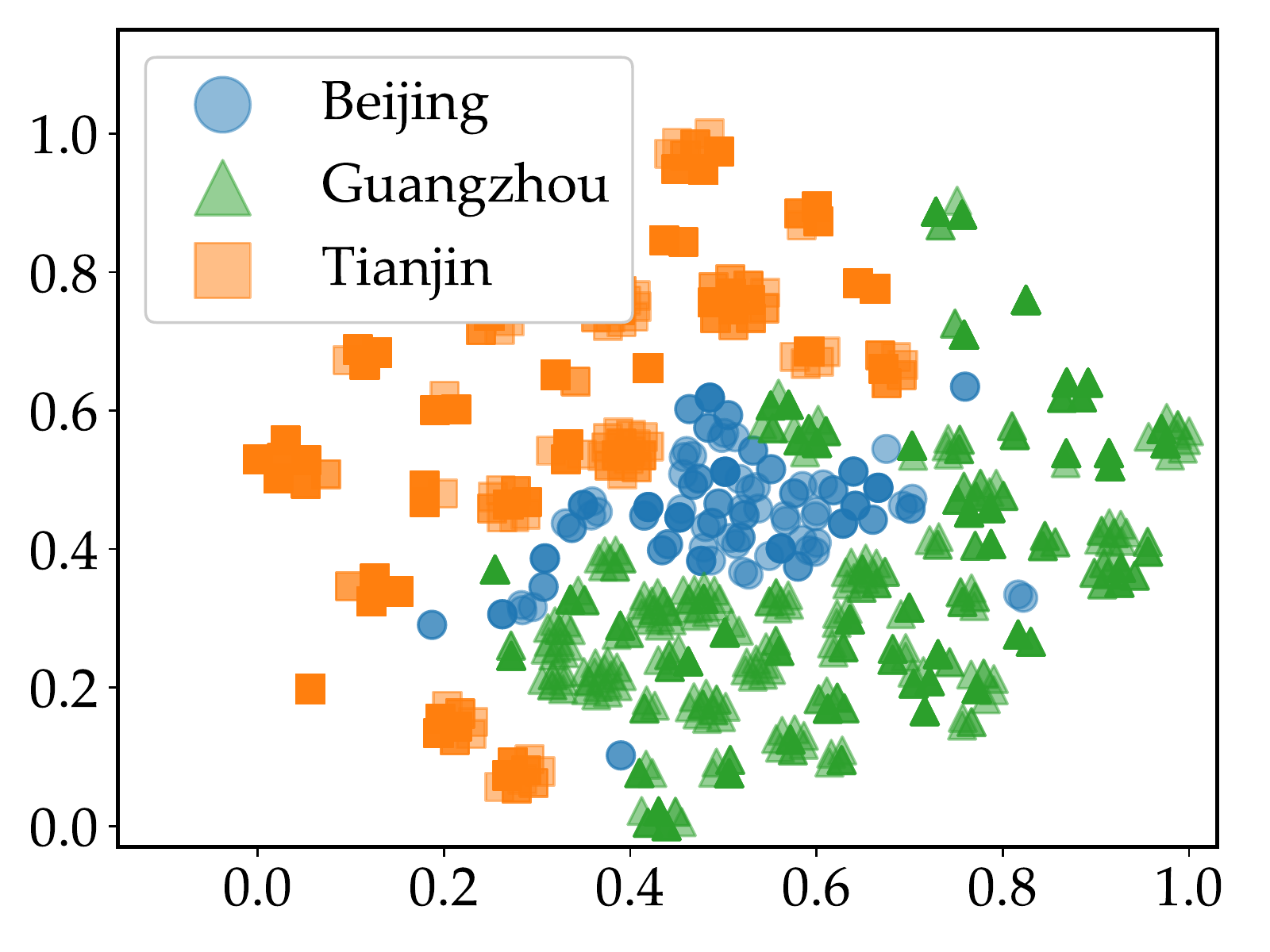}
		\subcaption{TSNE visualization}
		\label{fig_tsne}
	\end{subfigure}
	\caption{Domain analysis between cities}
	\label{fig:domain_analysis}
	\vspace{-10pt}
\end{figure}

\subsubsection{Domain Analysis between Cities}

To analyze the domain shift problem, we use the maximum mean discrepancy (MMD) \cite{pan2010domain} to quantify the difference between feature domains from the source city $SC$ and the target city $TC$, which maps the features into the reproducing kernel Hilbert space (RKHS) $\mathcal{H}$ \cite{smola2007hilbert} and calculates the square distance between the means of the embedded features:
\begin{equation}
	\text{MMD}(SC,TC)=\left \| \frac{1}{m_{s}}\sum_{i=1}^{m_{s}}\phi(s_{i}) - \frac{1}{m_{t}}\sum_{j=1}^{m_{t}}\phi(t_{j}) \right \|^{2}_{\mathcal{H}},
\end{equation}
where $s_{i}$ and $t_{j}$ are training samples from the source city and target city, $m_s$ and $m_t$ are the numbers of training samples, and $\phi(\cdot)$ is the kernel function.

We estimate the MMD for three cities, Beijing (BJ), Guangzhou (GZ) and Tianjin (TJ) in China, as shown in Fig. \ref{fig_mmd}.
The black solid lines in Fig. \ref{fig_mmd} are the rejecting thresholds for the null hypothesis test with power $\delta=0.05$. For all the three city pairs, the MMD results are much larger than the threshold, confirming that there exists a domain shift problem. Furthermore, we use the TSNE visualization \cite{maaten2008visualizing} to show the feature distributions of three cities, which reduces the feature dimension to 2.
As shown in Fig. \ref{fig_tsne}, Beijing and Guangzhou have more similar feature distribution, probably because they have closer city scale, EV development level and strategy (they develop EVs earlier and deploy more slow chargers, as shown later in Table \ref{table:datasets}).
By contrast, there is a larger feature difference between Tianjin and the other two cities, and the corresponding MMD values are also larger, probably because Tianjin develops EVs later and has a more different deployment strategy (more fast chargers).
In summary, both MMD and TSNE results demonstrate the necessity of designing a city domain adaptation approach to address the domain shift issue.

\begin{figure*}
	\centering
	\includegraphics[width=0.95\linewidth]{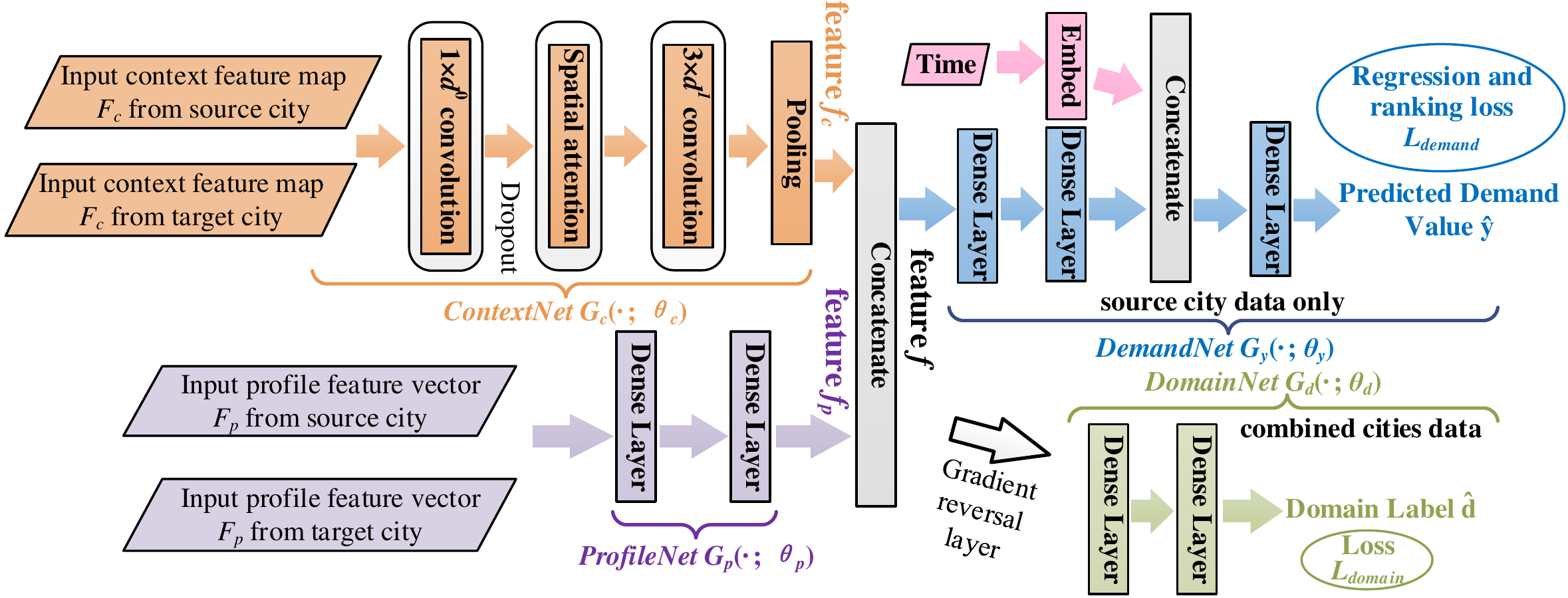}
	\caption{The architecture of \textit{AST-CDAN}}
	\label{fig:archi}
\end{figure*}
\subsection{Attention-based Spatial-Temporal City Domain Adaptation Network}
\label{subsec:AST-CDAN}
Figure \ref{fig:archi} shows the architecture of \textit{AST-CDAN}, consisting of four components: 
1) \textit{ContextNet} integrates convolution and spatial attention to model the influences from context features;
2) \textit{ProfileNet} learns latent features from the raw profile features by fully-connected layers;
3) \textit{DemandNet} is fed with the concatenation of outputs from \textit{ContextNet} and \textit{ProfileNet}, and integrates the temporal information to predict charging demands over different time intervals of one day;
4) \textit{DomainNet} guides the network to promote the features from \textit{ProfileNet} and \textit{ContextNet} to deeper domain-invariant representations for domain adaptation.
For convenience, let $S_{SC}$ and $S_{TC}$ denote the sets of training instances from source city and target city, respectively.

\subsubsection{ContextNet $G_c$}
It takes a feature map $F_c \in \mathbb{R}^{\lambda \times d}$ as input, which contains context features from $\lambda$ stations (itself and  $\lambda - 1$ nearest neighbor stations); $d$ is the dimension of context features.
We employ convolutional blocks to model the effects of context features. 
Each convolutional block contains one convolution layer, one batch normalization layer and one ReLU activation function:
\begin{equation}
	F_c^{out} = ReLU(BN(W_{c} * F_c^{in} + b_{c}))
\end{equation}
where $W_{c}$ and $b_{c}$ are learnable parameters, * represents the convolutional operation, and $BN$ means batch normalization \cite{ioffe2015batch}. 
To prevent overfitting, dropout \cite{JMLR:v15:srivastava14a} 
is applied after the first convolutional block.
Note that, different context features have inconsistent importance to the charging demands.
Therefore, we further use the spatial-wise attention model (SAM) \cite{gao2019scar} to encode the spatial dependencies, the details of which are illustrated in Fig. \ref{fig_sa}. 

\begin{figure}
	\centering
	\includegraphics[width=\linewidth]{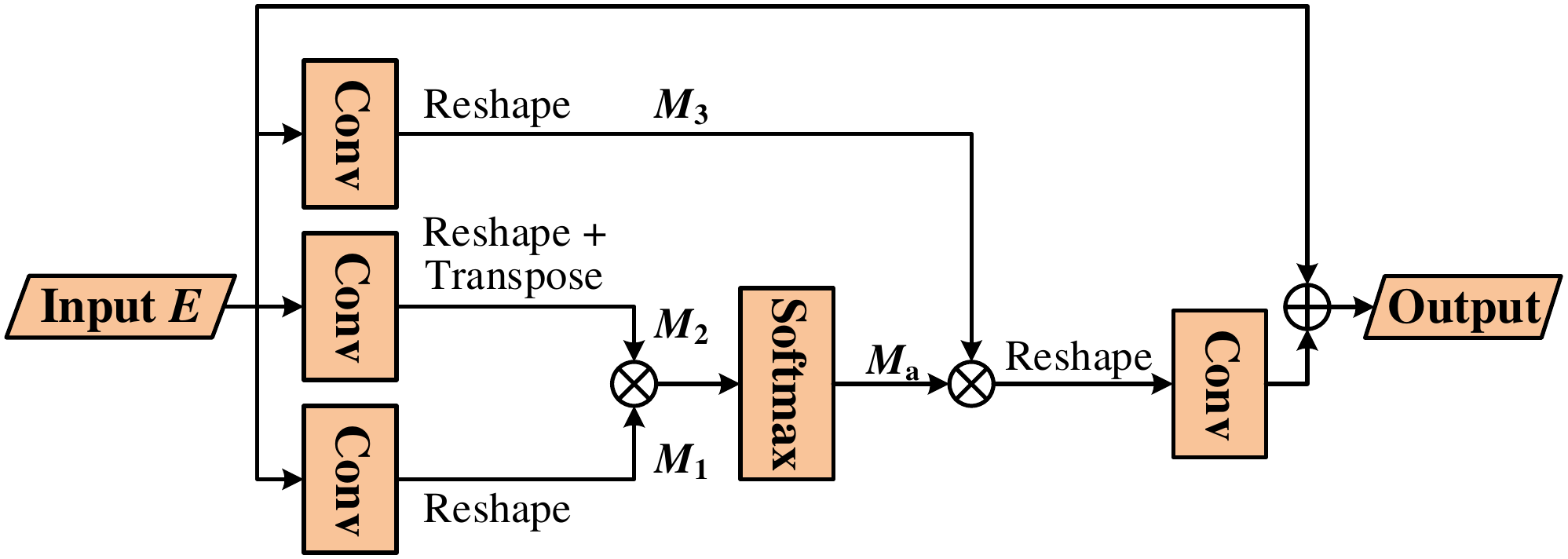}
	\caption{The architecture of Spatial Attention.} \label{fig_sa}
	\vspace{-10px}
\end{figure}

The spatial attention block takes $E$ as input to three $1 \times 1$ convolutional layers. After the reshape (and transpose) operations, we get three vectors $M_{1} \in \mathbb{R}^{HW \times 1}$, $M_{2} \in \mathbb{R}^{1 \times HW}$ and $M_{3} \in \mathbb{R}^{HW \times 1}$. $M_{1}$ and $M_{2}$ go through the matrix multiplication and softmax operations to get the spatial attention map $M_a \in \mathbb{R}^{HW \times HW}$. Then, we apply a matrix multiplication for $M_a$ and $M_3$, and reshape the output back to the size of $H \times W$. After one convolutional layer, we sum $E$ and the output to get $E_{a}$, which captures the effects of the contextual information on the original feature map. This process can be formulated as:
\begin{equation}
	M^{ji}_{a} = \frac{exp(M_{1}^{i} \cdot M_{2}^{j})}{\sum_{i=1}^{HW}exp(M_{1}^{i} \cdot M_{2}^{j})}
\end{equation}
\begin{equation}
	E_{a} = W_{a} * \text{vec}^{-1}_{H,W}(M_{a} \times M_{3}) + b_{a} + E
\end{equation}
where $W_{a}$ and $b_{a}$ are learnable parameters, * represents the convolution operation, $\text{vec}^{-1}_{H,W}$ means reshaping vector to matrix in shape of $H \times W$, and $M^{ji}_{a}$ means the influence of the value in the $i^{th}$ position on the value in the $j^{th}$ position.

The output of SAM is fed into the second convolutional block to enhance the performance. 
Finally, we apply the global average pooling operation on the output to get the final context feature $\textit{\textbf{f}}_{c}$.

\subsubsection{ProfileNet $G_p$}

It takes the raw profile features $F_p$ as input, and utilizes two fully-connected layers, each with a ReLU activation function. After that, we get the station profile feature $\bm{f}_{p}$.

The context feature $\textit{\textbf{f}}_c$ and the profile feature $\textit{\textbf{f}}_p$ are concatenated to obtain the final station feature $\textit{\textbf{f}}$, 
which will be fed into the \textit{DemandNet} and the \textit{DomainNet} simultaneously.

\subsubsection{DemandNet $G_y$}
The \textit{DemandNet} aims to predict the charging demand in each time interval. 
We use an embedding layer to transform the time into a vector $q$. 
Meanwhile, the feature $\textit{\textbf{f}}$ is fed into two fully-connected layers, and the output is concatenated with $q$ to get the hidden feature $H_{y}$. Finally, we use one fully-connected layer to get the predicted demand $\hat{y}$.

A direct method to optimize the \textit{DemandNet} is to minimize the regression loss over $S_{SC}$.
Inspired by \cite{liu2018will, sculley2010combined}, we find that the ranking loss is beneficial to enhance the regression prediction accuracy (Sect. \ref{subsec:prediction:performance}).
Thus, we combine the regression loss and the ranking loss for the \textit{DemandNet}, using a hyperparameter $\alpha$:
\begin{equation}
	L_{demand} = (1-\alpha) L_{reg} + \alpha L_{rank},
\end{equation}
where $L_{reg}$ is the mean square error between the predicted value $\hat{y}$ and the ground truth $y$ in $S_{SC}$:
\begin{equation}
	L_{reg}=\frac{1}{|S_{SC}|}\sum_{X \in S_{SC}}(\hat{y} - y)^{2}.
\end{equation}
We define $o_{ij} = y_{i} - y_{j}$ for the instance $i$ and $j$, which satisfy $y_{i} > y_{j}$. 
Thus, the probability that instance $i$ is listed higher than $j$ can be defined as $P_{ij} = \frac{e^{o_{ij}}}{1 + e^{o_{ij}}}$. 
Likewise, the predicted probability is $\hat{P}_{ij}$. 
Thus, we can use the cross entropy function to define the $L_{rank}$:
\begin{equation}
	L_{rank}=\frac{\sum_{i,j \land i\neq j} -P_{ij}log\hat{P_{ij}}-(1-P_{ij})log(1-\hat{P_{ij}})}{|S_{SC}|(|S_{SC}|-1)}
\end{equation}

\subsubsection{DomainNet $G_d$}
One way to solve the domain shift problem is to map the feature spaces of the source and target cities to the same space. Inspired by the previous study \cite{liu2018will}, we introduce the domain adaptation network to \textit{AST-CDAN}. The \textit{DomainNet} takes $\textit{\textbf{f}}$ as input and outputs a domain label that indicates which domain the feature belongs to. It contains two fully-connected layers:
\begin{equation}
	H_{d}^{1} = ReLU(W_{d}^{1} \textit{\textbf{f}} + b_{d}^{1})
\end{equation}
\begin{equation}
	\hat{d} = softmax(W_{d}^{2} H_{d}^{1} + b_{d}^{2})
\end{equation}
where $W_{d}^{1}$ and $W_{d}^{2}$ are the weight parameters, $b_{d}^{1}$ and $b_{d}^{2}$ are the bias parameters, and $\hat{d}$ is the predicted domain label.

We use the binary cross-entropy loss $L_{domain}$ to optimize the domain discrimination component:
\begin{equation}
	L_{domain} = \frac{1}{|S|} \sum_{X \in S} - d log\hat{d} - (1-d) log(1-\hat{d})
\end{equation}
where $d$ is the domain label and $S=S_{SC} \cup S_{TC}$.

\subsubsection{Optimization}
Based on the above components, we design the joint loss function composed by $L_{reg}$, $L_{rank}$ and $L_{domain}$. 
The \textit{DemandNet} needs to minimize $L_{reg}$ and $L_{rank}$ to improve the demand prediction performance. 
	The \textit{DomainNet} needs to minimize $L_{domain}$ for the domain classification. 
	However, the \textit{ContextNet} and \textit{ProfileNet} aim to 
	minimize $L_{reg}$ and $L_{rank}$ while maximizing $L_{domain}$, 
	because their goal is to produce domain-invariant feature representation 
	that is indistinguishable across domains.
	The optimization of the above components can be done with the following gradient updates:
	\begin{equation}\label{equ:gradient_station}
		\theta_s = \theta_s - \gamma \left(\alpha \dfrac{\partial L_{reg}}{\partial \theta_s} \!+\! \left(1\!-\!\alpha\right)\dfrac{\partial L_{rank}}{\partial \theta_s} \!-\! \beta\dfrac{\partial L_{domain}}{\partial \theta_s}\right)
	\end{equation}
	\begin{equation}
		\theta_y = \theta_y - \gamma \left(\alpha \dfrac{\partial L_{reg}}{\partial \theta_y} + \left(1-\alpha\right)\dfrac{\partial L_{rank}}{\partial \theta_y}\right)
	\end{equation}
	\begin{equation}
		\theta_d = \theta_d - \gamma \dfrac{\partial L_{domain}}{\partial \theta_d}
	\end{equation}
where $\theta_s$ are parameters of \textit{ContextNet} and \textit{ProfileNet}; $\theta_y$ are parameters of \textit{DemandNet}; $\theta_d$ are parameters of \textit{DomainNet}.

In Eq. (\ref{equ:gradient_station}), the gradients of $L_{reg}$, $L_{rank}$ and $L_{domain}$ are subtracted, which is different with summation in normal stochastic gradient descent (SGD) updates.
Accordingly, we add the gradient reversal layer \cite{ganin2015unsupervised} before the 
\textit{DomainNet}, which multiples the gradient from the \textit{DomainNet} by $-\beta$ 
during backward propagation. As a result, the joint loss function is defined as:
\begin{equation}
	L = (1-\alpha)L_{reg} + \alpha L_{rank} - \beta L_{domain}
\end{equation}

%% file: 4planning.tex
\section{Charger Planning}
\label{sec:4planning}
In this section, we first present the \textit{TIO} algorithm and then elaborate how to fine-tune the charger plan, following by the algorithm complexity analysis.

\subsection{Transfer Iterative Optimization}
\label{subsec:tio}

Hindered by the unacceptable complexity $\binom{B + 2|C_{TC}|-1}{2|C_{TC}|-1}$ of the straightforward approach in Sect. \ref{subsec:chllanges}, we adopt a heuristic strategy.
Generally, \textit{TIO} starts from a naive charger plan and iteratively fine-tunes the current charger plan toward a higher revenue.
In each iteration, we scale down the complexity by 1) decomposing the whole searching space into a small-scale collection of 5-element fine-tuned charger plan sets, and 2) only training the \textit{AST-CDAN} once.
In this way, the time complexity is proportional to the required number of iterations with a constant upper bound (Theorem \ref{theorem:iterupper}).

Specifically, given a charger plan $\mathcal{N}_{TC}$ in the target city, where station $c_i$'s charger plan is $(n_i^S, n_i^F)$, the fine-tuned charger plans are obtained as follows: 1) extending $c_i$'s charger plan to a fine-tuned charger plan set $N_i = \{(n_i^S, n_i^F)$, $(n_i^S+1, n_i^F)$, $(n_i^S-1, n_i^F)$, $(n_i^S, n_i^F+1)$, $(n_i^S, n_i^F-1)\}$, and 2) obtaining the collection of fine-tuned charger plan sets in the target city as $\bm{\mathcal{N}}'_{TC} = \{N_i|i = 1, \cdots, |C_{TC}|\}$.
In this way, $5^{|C_{TC}|}$ new plans could be constructed from $\bm{\mathcal{N}}'_{TC}$. 
If we re-train the \textit{AST-CDAN} and predict the charging demand for each plan, it will require to respectively conduct model training and prediction operations for $5^{|C_{TC}|}$ times, which is still impractical.
To address this issue, we further adopt two strategies to reduce the time complexity:

1) The \textit{AST-CDAN} is trained only once in each iteration, taking the current plan $\mathcal{N}_{TC}$ as the input. 
The features for fine-tuned plans only have slight difference on number of chargers, compared with the current plan, implying the versatility of the \textit{AST-CDAN} trained with $\mathcal{N}_{TC}$.
In each iteration, this strategy reduces the number of trainings from $5^{|C_{TC}|}$ to 1.

2) For each fine-tuned plan of any station $c_i$, we fix the features of the nearby stations the same as those extracted from the current plan $\mathcal{N}_{TC}$, and only use the new features of station $c_i$ to fed into the model trained with $\mathcal{N}_{TC}$, outputting the demand prediction results.
The features of nearby stations only have slight difference for those fine-tuned plans. This feature maintenance affects prediction results by 0.82\%, 0.22\%, and 1.46\% in transfer cases of BJ$\rightarrow$GZ, BJ$\rightarrow$TJ, and GZ$\rightarrow$TJ respectively.
Thus, in each iteration, this strategy reduces the number of prediction operations from $5^{|C_{TC}|}$ to 5, while ensuring almost the same prediction accuracy.

For convenience, the predicted demands for any fine-tuned plan set $N_i$ are denoted by $\widehat{\Gamma}_i = \{(\gamma_{ijt}^S, \gamma_{ijt}^F)| j = 1, \cdots, 5, t = 1, \cdots, T\}$ and the predicted demands for all the fine-tuned plan sets in all the stations are denoted by $\widehat{\bm{Y}}'_{TC} = \{\widehat{\Gamma}_i|i = 1, \cdots, |C_{TC}|\}$. The simplified prediction operation $f'$ is defined as:
\begin{equation}
	\widehat{\bm{Y}}'_{TC}=f'(f(\mathcal{N}_{SC}, \cdot, \bm{D}_{SC}, \bm{D}_{TC}, \bm{Y}_{SC}), \mathcal{N}_{TC}, \bm{\mathcal{N}}'_{TC})
\end{equation}
where $f(\mathcal{N}_{SC}, \cdot, \bm{D}_{SC}, \bm{D}_{TC}, \bm{Y}_{SC})$ is a predictor trained with $\mathcal{N}_{TC}$ and used for outputting the prediction results for any fine-tuned plan in $\bm{\mathcal{N}}'_{TC}$.
By now, we have obtained $\bm{\mathcal{N}}'_{TC}$ and $\widehat{\bm{Y}}'_{TC}$, so the remaining problem is how to update $\mathcal{N}_{TC}$ by selecting a charger plan from $N_i$ for each station $c_i$, so that the total revenue is maximized under the budget constraint $B$, as we will elaborate in Sect. \ref{subsec:optimization}.
Note that, there are biases in predicted demands, caused by the drift of data between training and prediction.
Therefore, the updated plan will be confirmed by retraining \textit{AST-CDAN} and prediction again, which will further determine whether to stop the \textit{TIO} algorithm.

Algorithm \ref{algo} shows the pseudocode of the \textit{TIO} algorithm, which operates with five main steps:

\begin{enumerate}
	\item Initialize the revenue and construct a naive charger plan by evenly allocating budget to each charger type of each candidate station, as illustrated in Fig. \ref{fig:example:inital} (lines 1-2).
	\item Train the \textit{AST-CDAN} model with the current charger plan $\mathcal{N}_{TC}$ to predict the demands $\widehat{\bm{Y}}_{TC}$, and compute the revenue $R_{TC}$;
	if the increased revenue is not greater than a threshold $\theta$, then return the current plan (lines 4-8).
	\item Extend the current plan $\mathcal{N}_{TC}$ to the collection of fine-tuned plan sets $\bm{\mathcal{N}}'_{TC}$ (line 9), as illustrated in Fig. \ref{fig:example:iter1}.
	\item Predict the charging demands $\widehat{\bm{Y}}'_{TC}$ for the fine-tuned plan sets $\bm{\mathcal{N}}'_{TC}$ (line 10).
	\item Invoke the \textit{DP-MK} algorithm (Algo. \ref{alg:dpmk}) to update the current plan $\mathcal{N}_{TC}$ (line 11), as illustrated in Fig. \ref{fig:example:iter1}; then go to step 2).
\end{enumerate}

\begin{algorithm}[ht]
	\SetKwInOut{Input}{input}\SetKwInOut{Output}{output}
	\Input{$\mathcal{N}_{SC}, \bm{D}_{SC}, \bm{D}_{TC},\bm{Y}_{SC}, B$}
	\Output{$\mathcal{N}_{TC}$}
	\BlankLine
	Initialize the revenue: $R \leftarrow 0$\;
	Construct an initial plan $\mathcal{N}_{TC}$ by evenly allocating budget $B$ to each candidate station and charger type\;
	\While{$True$}{
		$\widehat{\bm{Y}}_{TC} \leftarrow f(\mathcal{N}_{SC}, \mathcal{N}_{TC}, \bm{D}_{SC}, \bm{D}_{TC}, \bm{Y}_{SC})$\;
		$R_{TC} \leftarrow \sum_{i=1}^{|C_{TC}|}\sum_{t=1}^{T}(\widehat{y_{it}^S} p_{it}^S n_i^S + \widehat{y_{it}^F} p_{it}^F n_i^F)$\;
		\If{$R_{TC} - R \leq \theta$}{
			$\textbf{return } \mathcal{N}_{TC}$\
		}
		$R \leftarrow  R_{TC}$\;
		Build the fine-tuned plan sets $\bm{\mathcal{N}}'_{TC}$ from $\mathcal{N}_{TC}$\;
		$\widehat{\bm{Y}}'_{TC} \leftarrow f'(f(\mathcal{N}_{SC}, \cdot, \bm{D}_{SC}, \bm{D}_{TC}, \bm{Y}_{SC}), \mathcal{N}_{TC}, \bm{\mathcal{N}}'_{TC})$\;
		$\mathcal{N}_{TC} \leftarrow \textit{DP-MK}(\bm{\mathcal{N}}'_{TC}, \widehat{\bm{Y}}'_{TC}, B)$\;
	}
	\caption{Transfer Iterative Optimization (\textit{TIO})}\label{algo}
\end{algorithm}

\begin{figure}[ht]
	\begin{subfigure}{0.493\linewidth}
		\includegraphics[width=\linewidth]{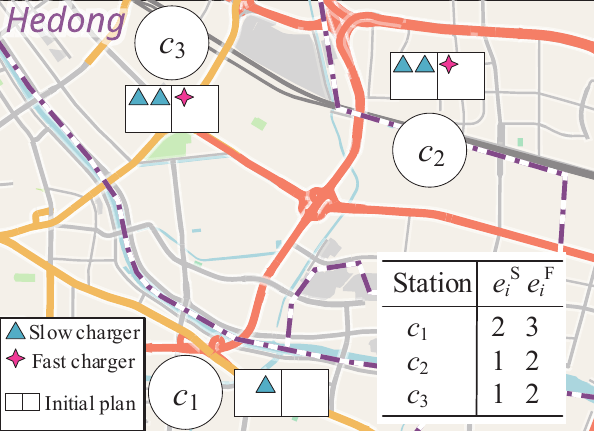}
		\subcaption{The initial plan generated by evenly allocating the budget ($B=12$)}
		\label{fig:example:inital}
	\end{subfigure}\hfill
	\begin{subfigure}{0.493\linewidth}
		\includegraphics[width=\linewidth]{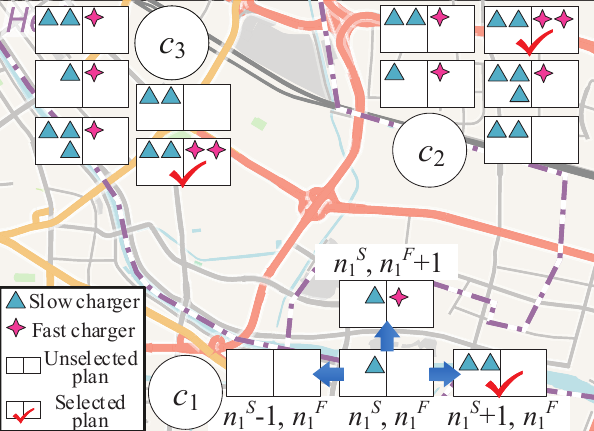}
		\subcaption{The fine-tuned plan sets and the updated plan in the first iteration}
		\label{fig:example:iter1}
	\end{subfigure}
	\caption{A running example of the \textit{TIO} algorithm}
	\label{fig:tio_example}
	\vspace{-15px}
\end{figure}

\subsection{Charger Plan Fine-tuning}
\label{subsec:optimization}
For convenience, let $N_i=\{(n_{ij}^S,n_{ij}^F)|j=1,\cdots,5\}$ denote the fine-tuned charger plan set of station $c_i$.
To optimize the plan $\mathcal{N}_{TC}$ toward higher revenue, it is required to solve the charger plan fine-tuning problem: given the collection of fine-tuned charger plan sets $\bm{\mathcal{N}}'_{TC}$, the predicted demands $\widehat{\bm{Y}}'_{TC}$ and the budget constraint $B$, the objective is to select one plan $(n_{ij}^S,n_{ij}^F)$ from $N_i$ for each station $c_i$, so that the total revenue is maximized while the total deployment cost of chargers does not exceed $B$.
In essence, the problem is an instance of the Multiple-choice Knapsack (MK) problem \cite{nauss19780}, formulated as follows:
\begin{equation}
	\begin{aligned}
		\max_{\nu_{ij}} \!\quad\! & 
		\sum_{i=1}^{|\textbf{C}_{TC}|}\sum_{j=1}^{5}\sum_{t=1}^{T}{(\gamma_{ijt}^S \!\cdot\! p_{it}^S \!\cdot\! n_{ij}^S \!\cdot\! \nu_{ij} \!+\! \gamma_{ijt}^F \!\cdot\! p_{it}^F \!\cdot\!  n_{ij}^F \!\cdot\! \nu_{ij})}\\
		\textrm{s.t.} \quad 
		& \sum_{i=1}^{|\textbf{C}_{TC}|}\sum_{j=1}^{5}c_i^S \!\cdot\! n_{ij}^S \!\cdot\! \nu_{ij} + c_i^F \!\cdot\! n_{ij}^F \!\cdot\! \nu_{ij} \leq B\\
		&  \sum_{j=1}^{5}\nu_{ij} = 1, \quad i=1,\cdots, |C_{TC}|\\
		& \nu_{ij} \in \{0, 1\}, \quad i=1,\cdots, |C_{TC}|, j=1,\cdots,5\\
	\end{aligned}
\end{equation}
where $\nu_{ij}$ is a binary decision variable, representing whether to choose the $j$-th fine-tuned plan $(n_{ij}^S,n_{ij}^F)$ for station $c_i$. 
The MK problem has been proven to be NP-complete, and it was pointed out that the dynamic programming approach performs well for a relatively small-scale problem \cite{nauss19780}.
Moreover, branch and bound algorithms with different relaxations could be used for providing approximate solutions while greatly reducing the time complexity \cite{nauss19780}.
In this work, we use a dynamic programming algorithm \textit{DP-MK} to obtain the optimal solution with the time complexity of $O(|C_{TC}|B)$.

Algorithm \ref{alg:dpmk} shows the pseudocode of \textit{DP-MK} algorithm, where 
\begin{itemize}
	\item $W[i][j]$ is the cost of the $j$-th fine-tuned plan of the $i$-th station;
	\item $V[i][j]$ is the daily revenue of the $j$-th fine-tuned plan of the $i$-th station;
	\item $R[i][k]$ is the maximum revenue under the budget of $k$, considering only the first $i$ stations;
	\item $S[i][k]$ records the optimal selection for the maximum revenue under the budget of $k$, considering only the first $i$ stations.
\end{itemize}

\begin{algorithm}
	\SetKwData{Left}{left}\SetKwData{This}{this}\SetKwData{Up}{up}
	\SetKwFunction{Union}{Union}\SetKwFunction{FindCompress}{FindCompress}
	\SetKwInOut{Input}{input}\SetKwInOut{Output}{output}
	\Input{$\bm{\mathcal{N}}'_{TC}$, $\widehat{\bm{Y}'}_{TC}$, $B$}
	\Output{$\mathcal{N}_{TC}$}
	
	\For{$i = 1, \cdots, |\textit{C}_{TC}|$}{
		\For{$j = 1, \cdots, 5$}{
			$W[i][j] \leftarrow e_i^S n_{ij}^S + e_i^F n_{ij}^F $\;
			$V[i][j]\!\leftarrow\!\sum_{t=1}^{T}{(\gamma_{ijt}^S  p_{it}^S  n_{ij}^S + \gamma_{ijt}^F  p_{it}^F n_{ij}^F)}$\;
		}
	}
	\For{$i=0, 1, \cdots, |C_{TC}|$}{
		\For{$k=0, 1, \cdots, B$}{
			$R[i][k] \leftarrow 0$\;
			$S[i][k] \leftarrow $ an empty list\;
		}
	}
	\For{$i=1, 2, \cdots, |C_{TC}|$}{
		\For{$j=1, 2, \cdots, 5$}{
			\For{$k=W[i][j], W[i][j]+1, \cdots, B$}{
				\If{$R[i][k] < R[i-1][k-W[i][j]] + V[i][j]$}{
					$R[i][k]\!\leftarrow\!R[i-1][k - W[i][j]] + V[i][j]$\;
					$S[i][k] \leftarrow S[i-1][k - W[i][j]]$\;
					Append $j$ to the tail of $S[i][k]$\;
				}
			}
		}
	}
	\textbf{return} $S[|C_{TC}|][argmax_k R[|C_{TC}|][k]]$\;
	\caption{\textit{DP-MK}}
	\label{alg:dpmk}
\end{algorithm}

\subsection{Algorithm Complexity Analysis}
As previously mentioned, the time complexity of the \textit{TIO} algorithm is proportional to the required number of iterations, with a constant upper bound as follows.
\begin{theorem}
	The required number of iterations for the \textit{TIO} algorithm has an upper bound $\frac{B}{\theta} \times \max_i\left(\frac{\sum_{t=1}^T p_{it}^S}{e_i^S}, \frac{\sum_{t=1}^T p_{it}^F}{e_i^F}\right)$.
	\label{theorem:iterupper}
\end{theorem}
\begin{proof}
	Let $u=\max_i\left(\frac{\sum_{t=1}^T p_{it}^S}{e_i^S}\right)$ and $v=\max_i\left(\frac{\sum_{t=1}^T p_{it}^F}{e_i^F}\right)$ denote the maximum revenues that per unit cost can produce by slow chargers and fast chargers among all the stations $\textit{C}_{TC}$.
	Further, let $w=\max(u, v)$ denote the maximum revenue that per unit cost can produce among any charger and any $c_i \in \textit{C}_{TC}$.
	There is an upper bound of revenue $R$: $\frac{B}{w}$. 
	Since the increased revenue is at least $\theta$ in each iteration, there is an upper bound of the number of iterations: $\frac{B}{w\cdot \theta}$ = $\frac{B}{\theta} \times \max_i\left(\frac{\sum_{t=1}^T p_{it}^S}{e_i^S}, \frac{\sum_{t=1}^T p_{it}^F}{e_i^F}\right)$.
\end{proof}

Then we analyze the time complexity of the \textit{DP-MK} algorithm. The number of loops in line 9 is $|C_{TC}|$; the number of loops in line 10 is 5; the number of loops in line 11 is $B$ (at most). As a result, the time complexity of \textit{DP-MK} algorithm is $O(|C_{TC}|B)$.

Finally we analyze the time complexities of the \textit{TIO} algorithm where four steps are processed in each iteration: the time complexity for model training in step 2) is a constant, about 1 hour; that in steps 3) and 4) is $O(|C_{TC}|)$; and that in step 5) is $O(|C_{TC}|B)$. 
There are at most $O(B)$ iterations.
Thus, the total time complexity of the \textit{TIO} algorithm is $O(|C_{TC}|B^2)$.

%% file: 5evaluation.tex
\section{Experimental Evaluation}
\label{sec:5evaluation}

\subsection{Experimental Settings}
\textbf{Datasets.} We collected the charging station data, including the locations, number of slow/fast chargers, service prices, and historical charging demands, from a public EV charging platform \textit{Star Charge}\footnote{https://en.starcharge.com}, which has the highest monthly usage in Chinese public EV charging market.
Meanwhile, we collected the POI and transportation data from AutoNavi\footnote{https://amap.com}, and collected the road network data from OpenStreetMap\footnote{https://www.openstreetmap.org}.
All the data are from three cities, Beijing, Guangzhou and Tianjin in China, and the charging demands are recorded during 8:00-21:00 every day from 05/12/2019 to 15/01/2020. Table \ref{table:datasets} shows the dataset details in each city. In addition, according to China's charging pile industry report \cite{CPreport}, we set $e_i^S$ and $e_i^F$ of each station as 33000 and 54000 in RMB. The radius $r$ used for feature extraction is set to 1 km. 

We mainly consider three cross-city prediction/planning tasks, BJ $\rightarrow$ GZ, BJ $\rightarrow$ TJ and GZ $\rightarrow$ TJ, which is in line with the development order and level of EV charging stations in three cites.

All the experiments are run in a Linux server (CPU: E5-2620 v4 @ 2.10GHz, GPU: NVIDIA Tesla P100).
For the \textit{AST-CDAN}, we use Pytorch to build it, and set $\alpha \in \left\lbrace 0, 0.3, 0.5, 0.8, 1.0\right\rbrace$; $\beta=0.1$; the batch size $bs=64$; the learning rate $lr \in \{0.01$, $0.005$, $0.001$, $0.0005$, $0.0001\}$. For the \textit{TIO}, we set $\theta=0.1$, $u^S=40$, and $u^F=20$.

\begin{table}[h]
	\centering
	\caption{Details of the datasets}
	\label{table:datasets}
	\begin{tabular}{lllllll}
		\hline
		& \multicolumn{6}{c}{City}                                                                  \\ \cline{2-7} 
		Data & \multicolumn{2}{l}{Beijing} & \multicolumn{2}{l}{Guangzhou} & \multicolumn{2}{l}{Tianjin} \\ \hline
		\# of charging stations  & 138 & & 123 & & 101 & \\
		\# of slow chargers  & 1473 & & 1434 & & 273 & \\
		\# of fast chargers  & 733 & & 608 & & 551 & \\
		\# of POIs & 576726 & & 503920 & & 362160 & \\
		\# of transportation facilities & 251758 & & 211756 & & 155996 & \\
		\# of roads & 841 & & 726 & & 651 & \\ \hline
	\end{tabular}
\end{table}

\subsection{Evaluation on Charger Demand Prediction}
\textbf{Baselines.} We compare our \textit{AST-CDAN} with three baselines:
\begin{itemize}
	\item \textit{LASSO (Least Absolute Shrinkage and Selection Operator)}, a well-known linear regression method that performs both variable selection and regularization to enhance the prediction accuracy;
	\item \textit{GBRT (Gradient Boost Regression Tree)}, a boosting method based on decision tree that can deal with heterogeneous data and has been widely used in many data mining tasks;
	
	\item \textit{MLP (Multi-layer Perceptron)}, a feedforward deep neural network with four full-connected layers and one ReLU activation function.
	
\end{itemize}

\noindent\textbf{Variants.} We also compare our \textit{AST-CDAN} with three variants:
\begin{itemize}
	\item \textit{AST-CDAN/AP}, which removes both the spatial attention and the \textit{ProfileNet} from \textit{AST-CDAN};
	\item \textit{AST-CDAN/P}, which removes the \textit{ProfileNet} from \textit{AST-CDAN};
	\item \textit{AST-CDAN/D}, which removes the \textit{DomainNet} from \textit{AST-CDAN}.
\end{itemize}

\noindent\textbf{Metric.}
One widely used metric, \textit{RMSE} (Root Mean Square Error), is adopted to evaluate the prediction performance.

\begin{table}[h]
	\setlength{\tabcolsep}{5pt}
	\centering
	\caption{Comparison results for charger demand prediction}

	\label{table:slow}
	\begin{threeparttable}[b]
		\begin{tabular}{c|cccccc}
			\hline
			& \multicolumn{2}{c}{BJ $\rightarrow$ GZ}                                   
			& \multicolumn{2}{c}{BJ $\rightarrow$ TJ}                                    
			& \multicolumn{2}{c}{GZ $\rightarrow$ TJ}                                    
			\\ \cline{2-7} 
			\multirow{-2}{*}{RMSE} 
			& \multicolumn{1}{c}{Slow} & \multicolumn{1}{c}{Fast} & \multicolumn{1}{c}{Slow} & \multicolumn{1}{c}{Fast} & \multicolumn{1}{c}{Slow} & \multicolumn{1}{c}{Fast}
			\\ \hline
			LASSO
			& 0.3052 & 0.2935 & 0.2999 & 0.3527 & 0.2995 & 0.3666 \\
			GBRT                     
			& 0.3077 & 0.2919 & 0.3378 & \underline{0.3321} & 0.3032 & \underline{0.3429} \\
			MLP
			& \underline{0.3010} & \underline{0.2914} & \underline{0.2906} & 0.3435 & \underline{0.2939} & 0.3560 \\
			\hline
			AST-CDAN & \textbf{0.2834} & \textbf{0.2860} & \textbf{0.2584} & \textbf{0.3234} & \textbf{0.2401} & \textbf{0.3264} \\ 
			Gains\tnote{1} & 5.85\% & 1.85\% & 11.08\% & 2.62\% & 18.31\% & 4.81\% \\
			\hline
		\end{tabular}
		\begin{tablenotes}
			\item[1] The performance gain is achieved by our \textit{AST-CDAN} compared with the best baseline (underlined).
		\end{tablenotes}
	\end{threeparttable}
\end{table}

\label{subsec:prediction:performance}

\noindent\textbf{Performance Comparisons}.
Table \ref{table:slow} compares our AST-CDAN with the three baselines for slow and fast charger demands in three city pairs. 
We observe that: (1) the deep learning methods (\textit{MLP} and \textit{AST-CDAN}) are superior to traditional regression methods (\textit{LASSO} and \textit{GBRT}), demonstrating the advantages of capturing non-linear correlations between features and charging demands; (2) our \textit{AST-CDAN} performs best due to the added ability of domain adaptation; (3) our \textit{AST-CDAN} has more gains in BJ$\rightarrow$TJ and GZ$\rightarrow$TJ, indicating its bigger superiority when the target city (Tianjin) has a more different feature distribution (as analyzed in Sect. \ref{subsec:feature}).

\begin{figure}[t]
	\begin{subfigure}{0.5\linewidth}
		\includegraphics[width=\textwidth]{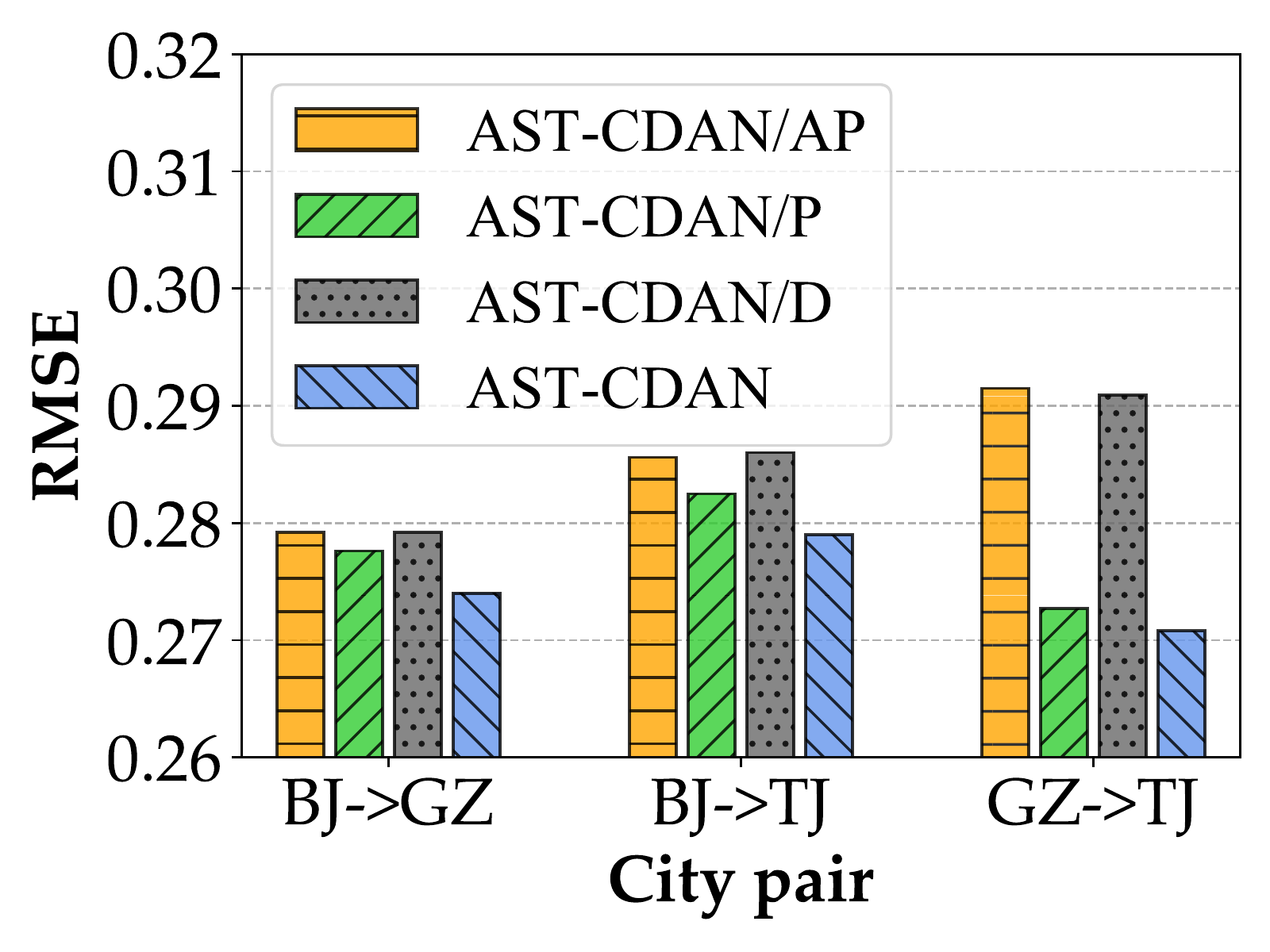}
		\subcaption{Comparisons with variants}
		\label{fig:evaluation:prediction:variant}
	\end{subfigure}%
	\begin{subfigure}{0.5\linewidth}
		\includegraphics[width=\textwidth]{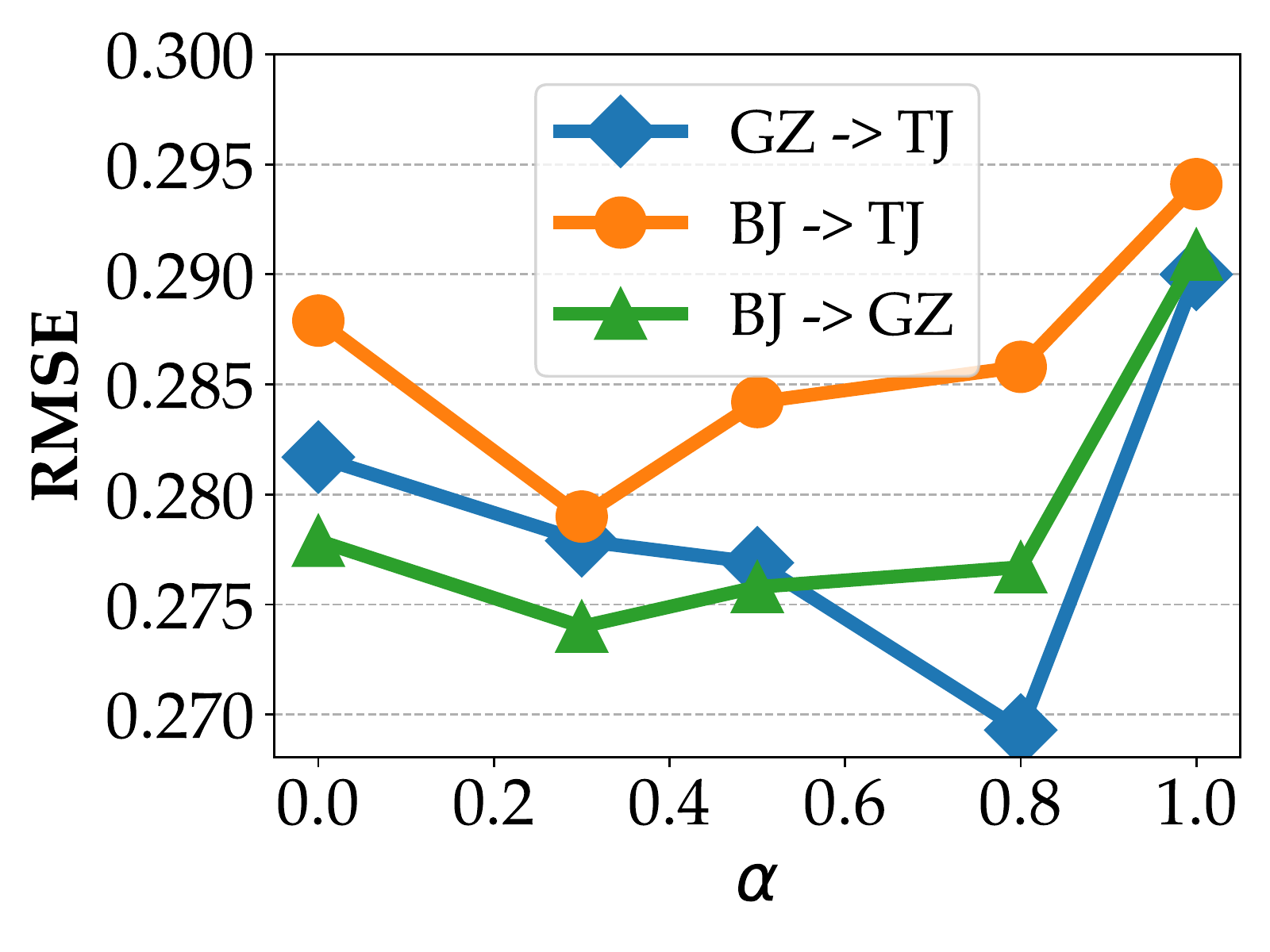}
		\subcaption{Effects of ranking loss weight}
		\label{fig:evaluation:prediction:alpha}
	\end{subfigure}

	\caption{Effects of different components on charging demand prediction}
	\label{fig:evaluation:prediction}
	\vspace{-10px}
\end{figure}

\noindent\textbf{Effect of Each Component}.
Figure \ref{fig:evaluation:prediction:variant} compares our \textit{AST-CDAN} with its three variants for the overall prediction results in three city pairs.
We observe that: (1) \textit{AST-CDAN/P} outperforms \textit{AST-CDAN/AP}, indicating the important role of the spatial attention on capturing contextual information; (2) \textit{AST-CDAN} outperforms \textit{AST-CDAN/P}, indicating the necessity of modeling profile features; (3) The RMSE of \textit{AST-CDAN/D} increases significantly compared with \textit{AST-CDAN}, implying a great negative influence of the domain shift problem, while \textit{AST-CDAN} has an obvious advantage on addressing the domain shift problem. 
To further illustrate whether \textit{DomainNet} can learn domain-invariant feature representations, we show the TSNE visualization results of feature distributions for \textit{AST-CDAN} with and without \textit{DomainNet} respectively in Fig. \ref{fig:evaluation:prediction:domain}.
It is obvious to see that the feature distribution becomes more consistent between source city (BJ) and target city (TJ) with \textit{DomainNet}. We also get a lower MMD value with \textit{DomainNet} (0.0575) than that without \textit{DomainNet} (0.9612).
In addition, Fig. \ref{fig:evaluation:prediction:alpha} shows the effect of ranking loss weight $\alpha$.
The optimal results are achieved always when $\alpha$ is equal to some intermediate value, implying that the ranking loss can help to enhance the prediction accuracy.

\begin{figure}[t]

	\begin{subfigure}{0.5\linewidth}
		\includegraphics[width=\linewidth]{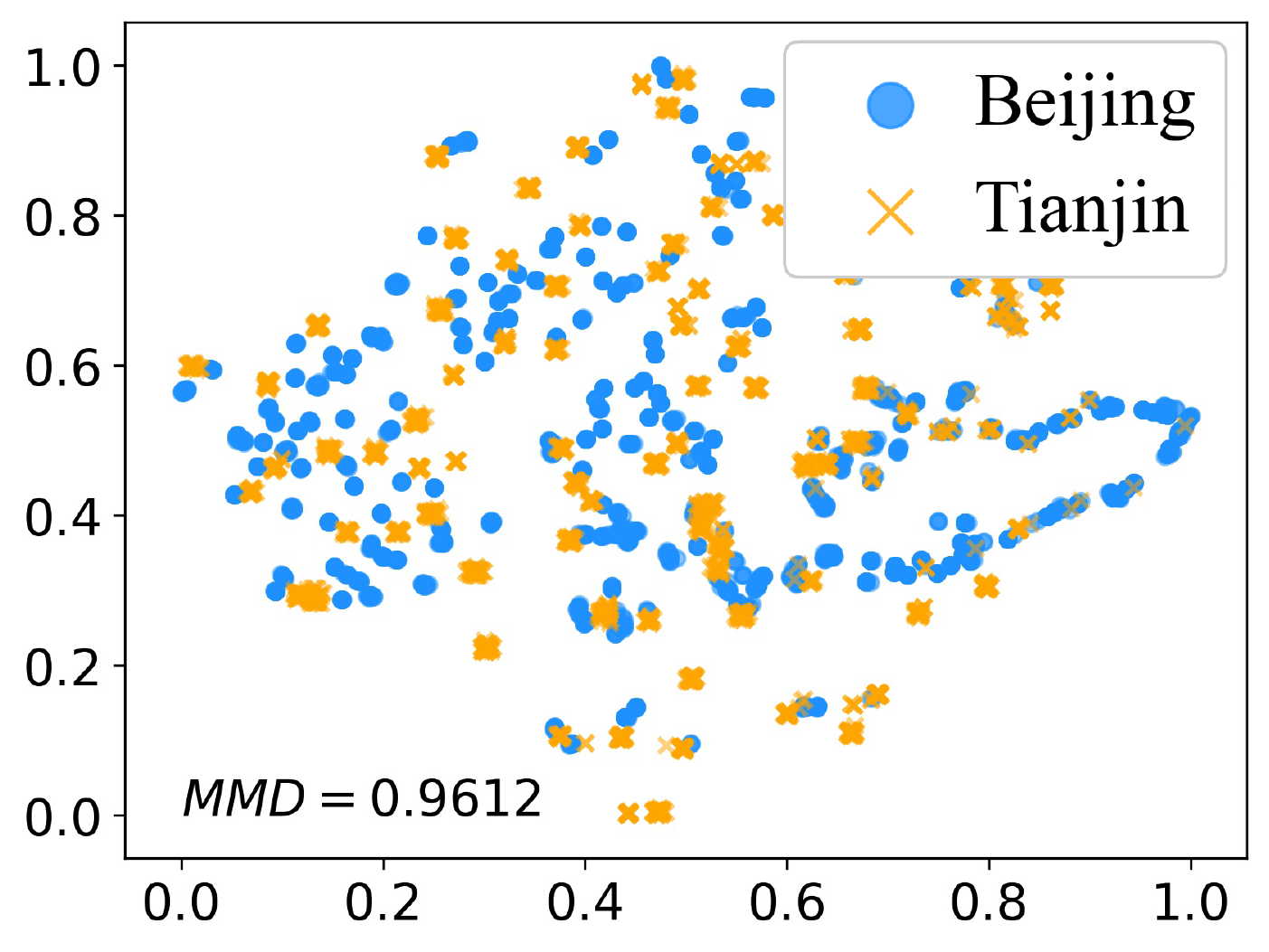}
		\subcaption{Features w/o the \textit{DomainNet}}
	\end{subfigure}%
	\begin{subfigure}{0.5\linewidth}
		\includegraphics[width=\linewidth]{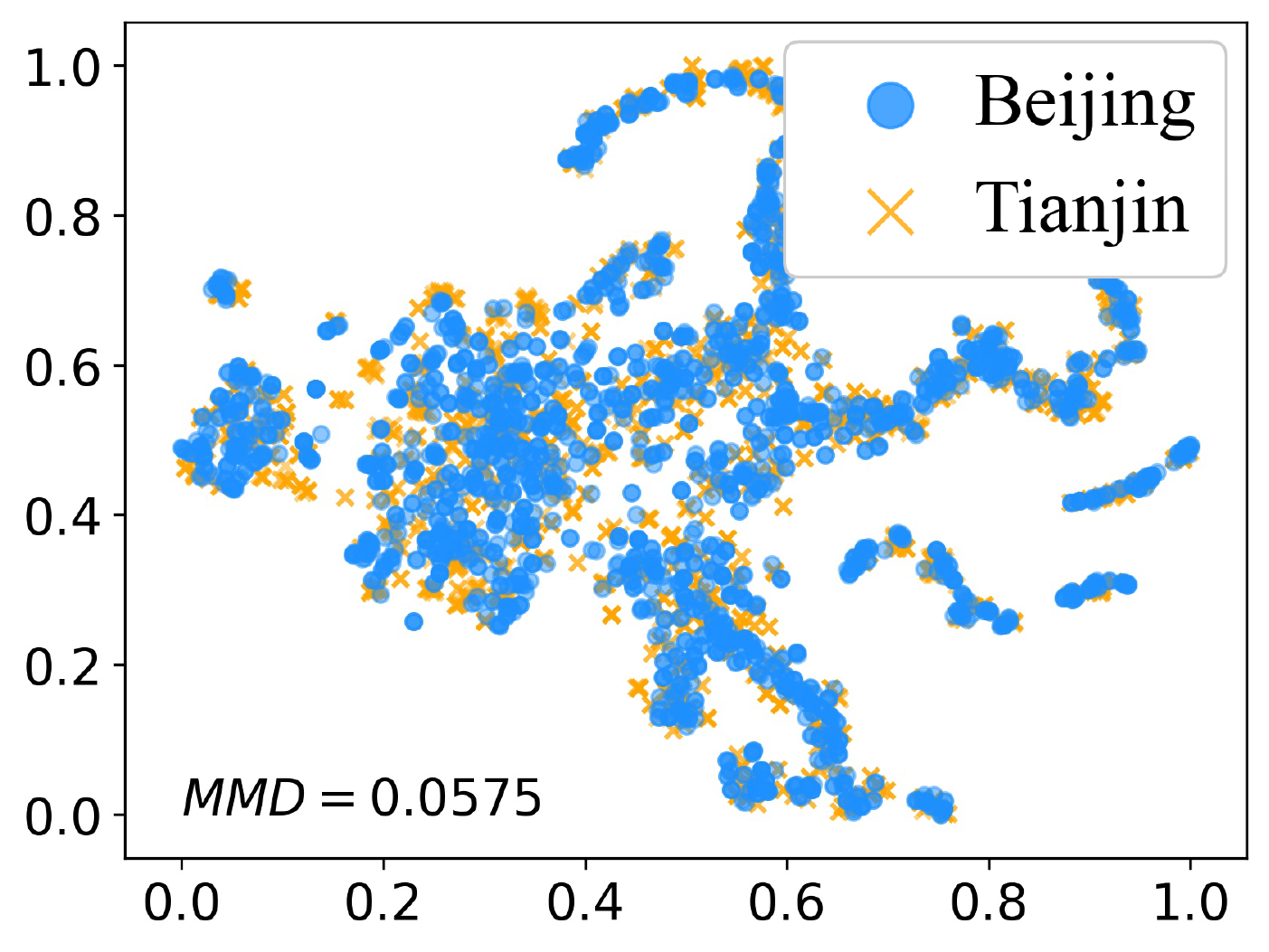}
		\subcaption{Features with the \textit{DomainNet}}
	\end{subfigure}
	\caption{Effect of \textit{DomainNet} on TSNE visualization results}
	\label{fig:evaluation:prediction:domain}
	\vspace{-10px}
\end{figure}

\subsection{Evaluation on Charger Planning}
\label{evaluation:planning}

\textbf{Baselines.} We compare our \textit{TIO} algorithm with four baselines:
\begin{itemize}

	\item \textit{Parking as Proxy (Park)}, which follows the work \cite{chen2013locating} to take parking sessions as the proxy of charging demands. We builds the Voronoi diagram by taking charging stations as seeds and aggregates parking sessions\footnote{The parking sessions data are provided by www.soargift.com} in each divided region as the charging demand. Then the budget is allocated to each charging station proportionally to the charging demand.
	\item \textit{Population as Proxy (Pop)}, which follows the work \cite{xiong2017optimal} to estimate the charging demand in proportion to the population of the region to which the charging station belongs. \textit{Population as Proxy} has the same process as \textit{Parking as Proxy} except that the population\footnote{The population map data are provided by www.worldpop.org} is used for estimation.
	\item \textit{Even}, which is a naive solution by evenly allocating the budget to each charger type of each charging station.
	\item \textit{Charger-based Greedy (CG)} algorithm \cite{du2018demand}, which assumes that the charging demands of all the stations are already known, and greedily places the charger in the candidate station with the maximum increased demand reward.
	In our experiments, we use the historical charging demands in the real world as inputs, although it is impractical in a new city.
\end{itemize}

We also compare the algorithms with the real-world EV charger plans (named as \textit{``Real''}) that have been deployed in the three cities. Specifically, we compute the total cost that is required to deploy the \textit{real} plan, and use it as the budget to determine charger plans with different algorithms for performance comparisons.

\noindent\textbf{Metrics.} We compare all the algorithms in terms of \textit{daytime revenue} (during 8:00-21:00 of one day). 
Besides, we evaluate the time complexity in terms of \textit{\# of trainings}.

\begin{figure}[t]
	\begin{subfigure}{0.5\linewidth}
		\includegraphics[width=\linewidth]{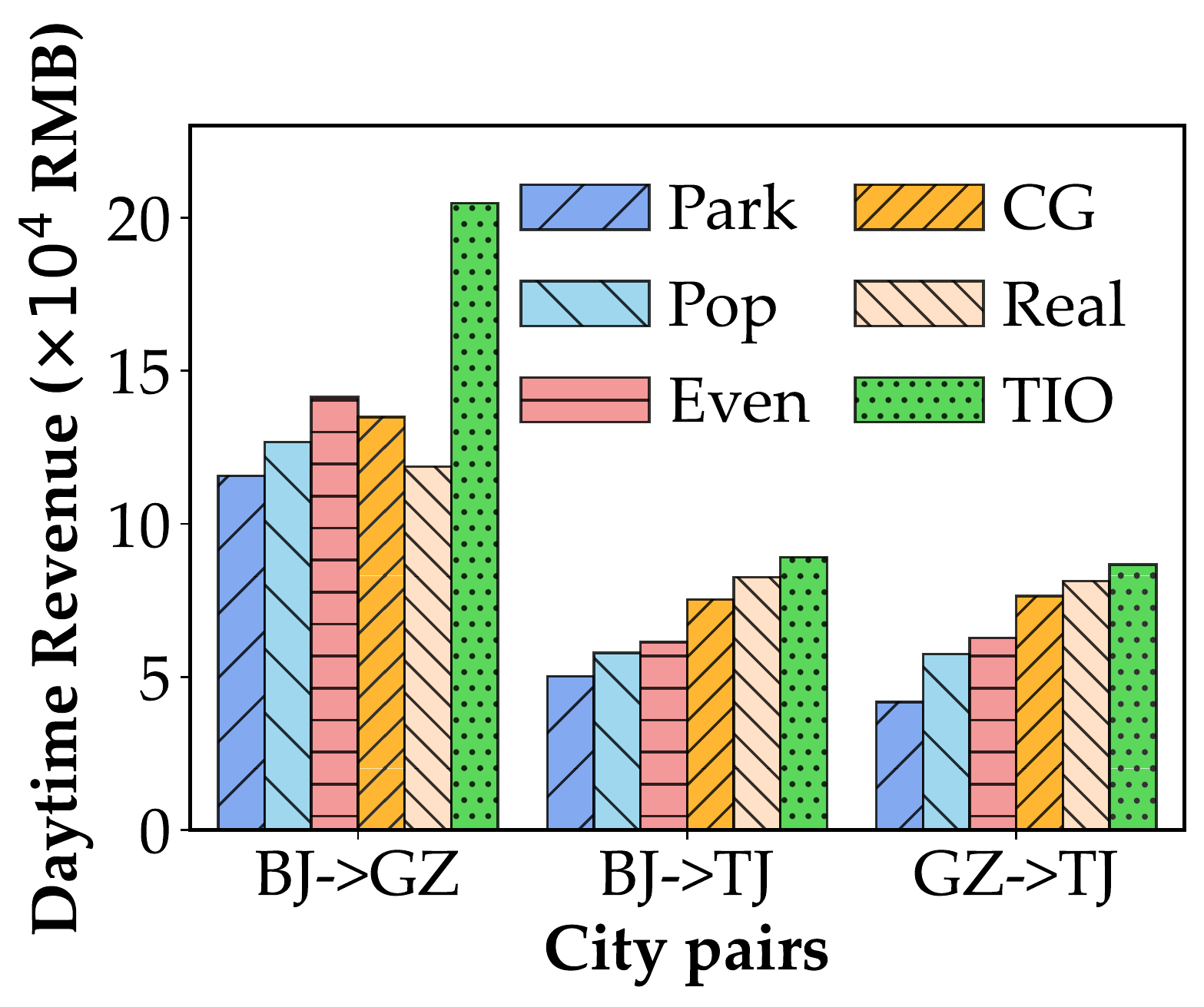}
		\subcaption{Comparison with real plans}
		\label{fig:planning:reality}
	\end{subfigure}%
	\begin{subfigure}{0.5\linewidth}
		\includegraphics[width=\linewidth]{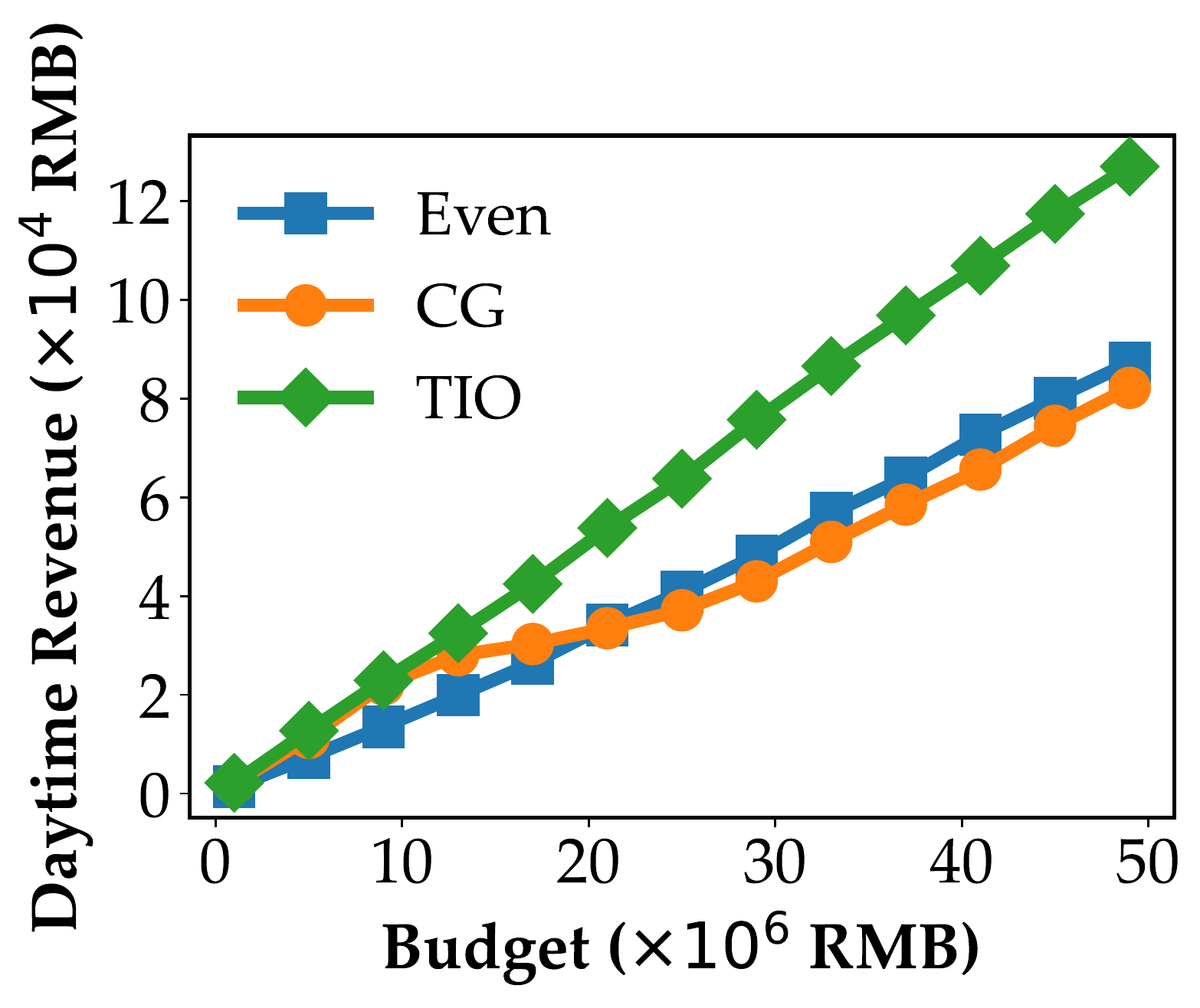}
		\subcaption{BJ $\rightarrow$ GZ}
		\label{fig:evaluation:planning:scalability:budget:b2g}
	\end{subfigure}
	\begin{subfigure}{0.5\linewidth}
		\includegraphics[width=\linewidth]{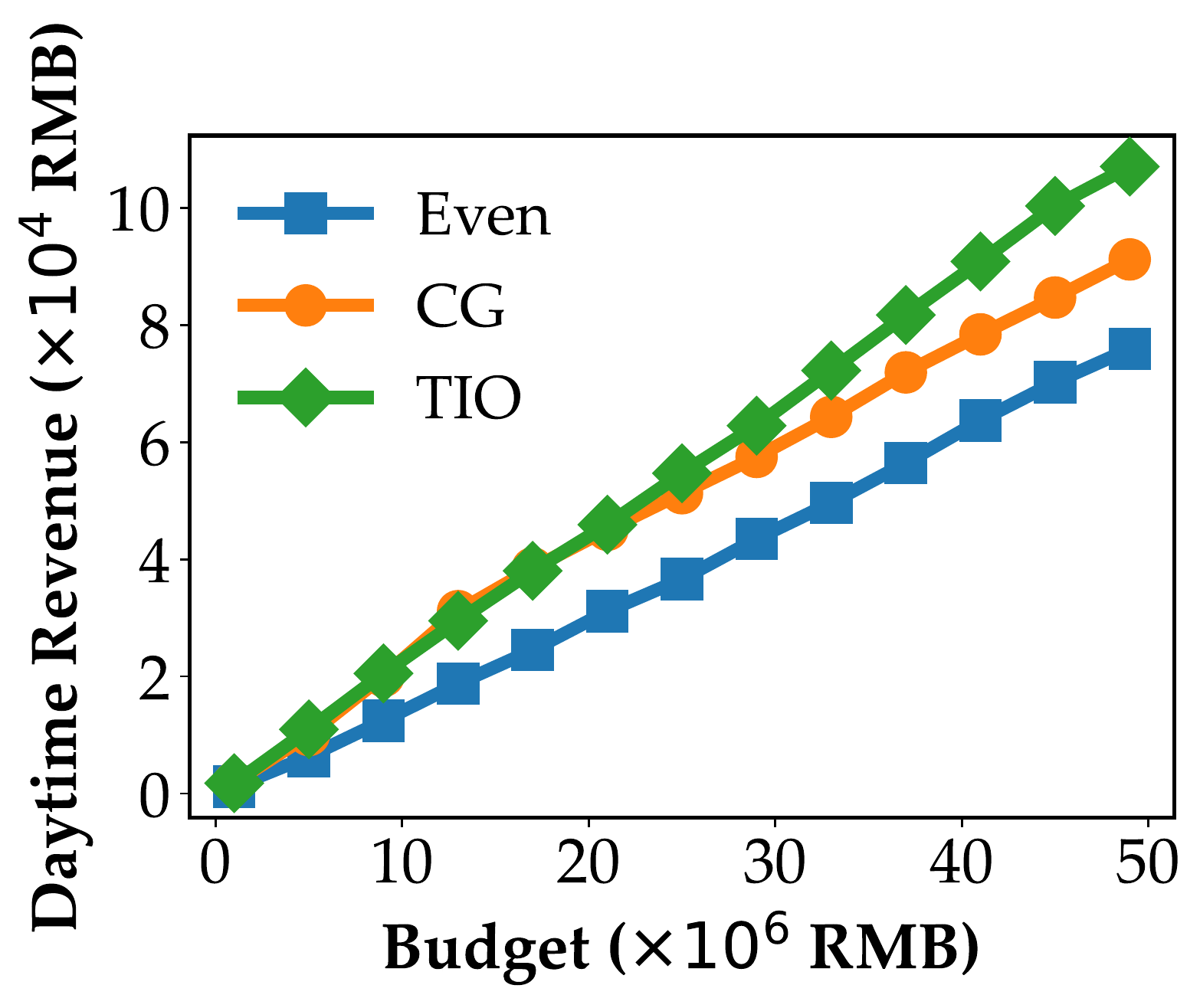}
		\subcaption{BJ $\rightarrow$ TJ}
		\label{fig:evaluation:planning:scalability:budget:b2t}
	\end{subfigure}%
	\begin{subfigure}{0.5\linewidth}
		\includegraphics[width=\linewidth]{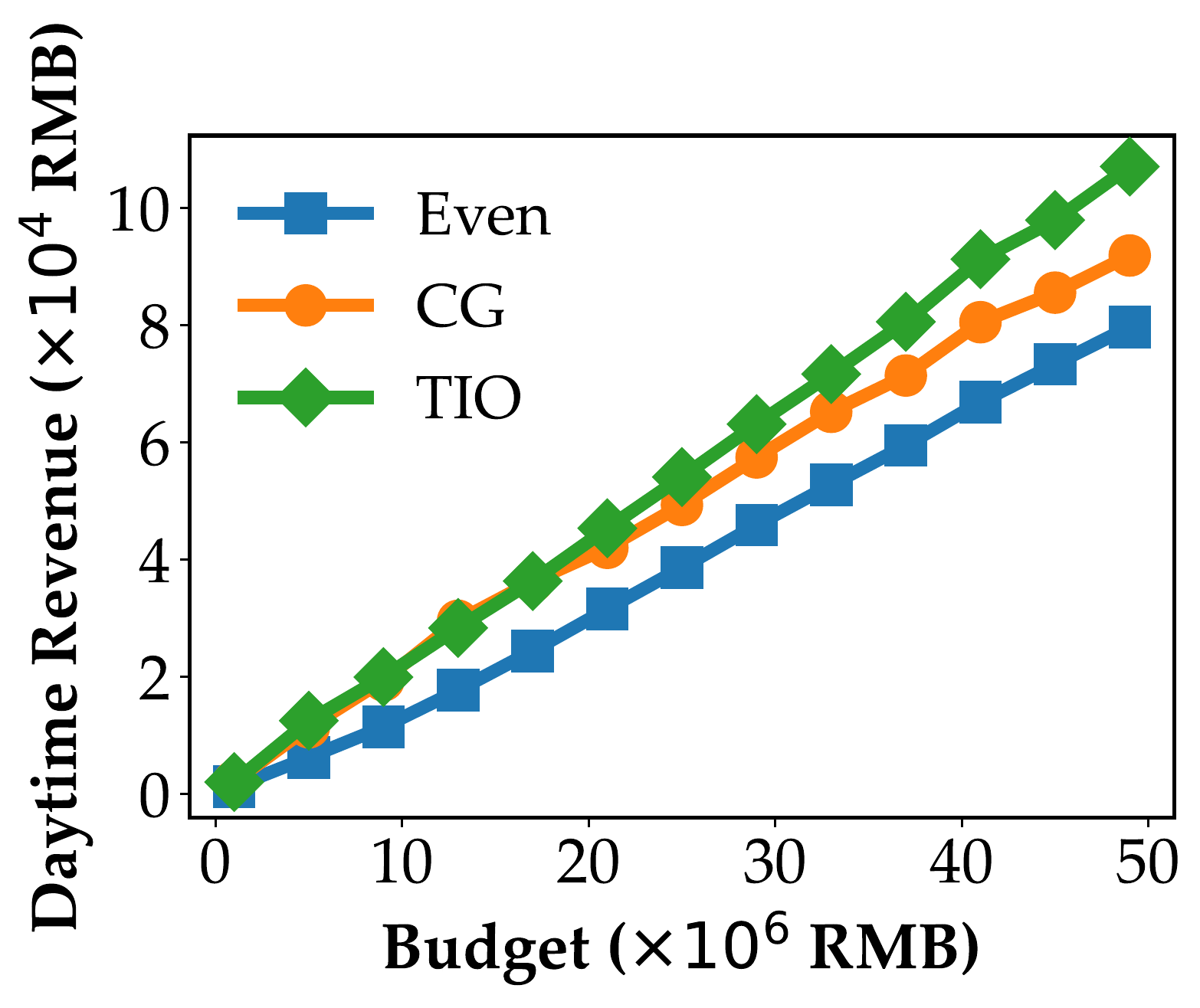}
		\subcaption{GZ $\rightarrow$ TJ}
		\label{fig:evaluation:planning:scalability:budget:g2t}
	\end{subfigure}
	\caption{Comparison results for charger planning}
	\label{fig:evaluation:planning:scale}
	\vspace{-10px}
\end{figure}

\noindent\textbf{Performance Comparisons with Real Plans}.
As shown in Fig. \ref{fig:planning:reality}, \textit{TIO} outperforms other baselines and achieves 72.5\%, 7.9\% and 6.7\% revenue increment comparing with the real plans in BJ$\rightarrow$GZ, BJ$\rightarrow$TJ and GZ$\rightarrow$TJ cases, respectively. The increment is smaller in Tianjin than that in Guangzhou, because 1) the lately deployed plan in Tianjin has a higher average utilization rate (45\%) than the early deployed plan in Guangzhou (31\%), and 2) the deployment scale and the used budget are smaller than that in Guangzhou. It implies that our \textit{TIO} can intelligently find a better charger plan with higher revenue, while avoiding poor charger plans, by efficiently utilizing the budget on chargers with higher demands. It also implies that \textit{TIO} can learn more useful knowledge from the data in other cities than the human experience.
Second, we observe that, \textit{CG} performs better than \textit{Even} in Tianjin but the results are just the opposite in Guangzhou. It is because that Guangzhou has more slow chargers, which guides \textit{CG} to spend more budget on satisfying the demands of slow chargers; while the slow chargers have a lower cost-benefit ratio in reality.
By contrast, our \textit{TIO} can avoid this drawback. 
Finally, \textit{Park} and \textit{Pop} get the worst results, which confirms that the implicit data are inappropriate to represent EV charging demands for charger planning.

\begin{figure}[t]
	\includegraphics[width=1.02\linewidth]{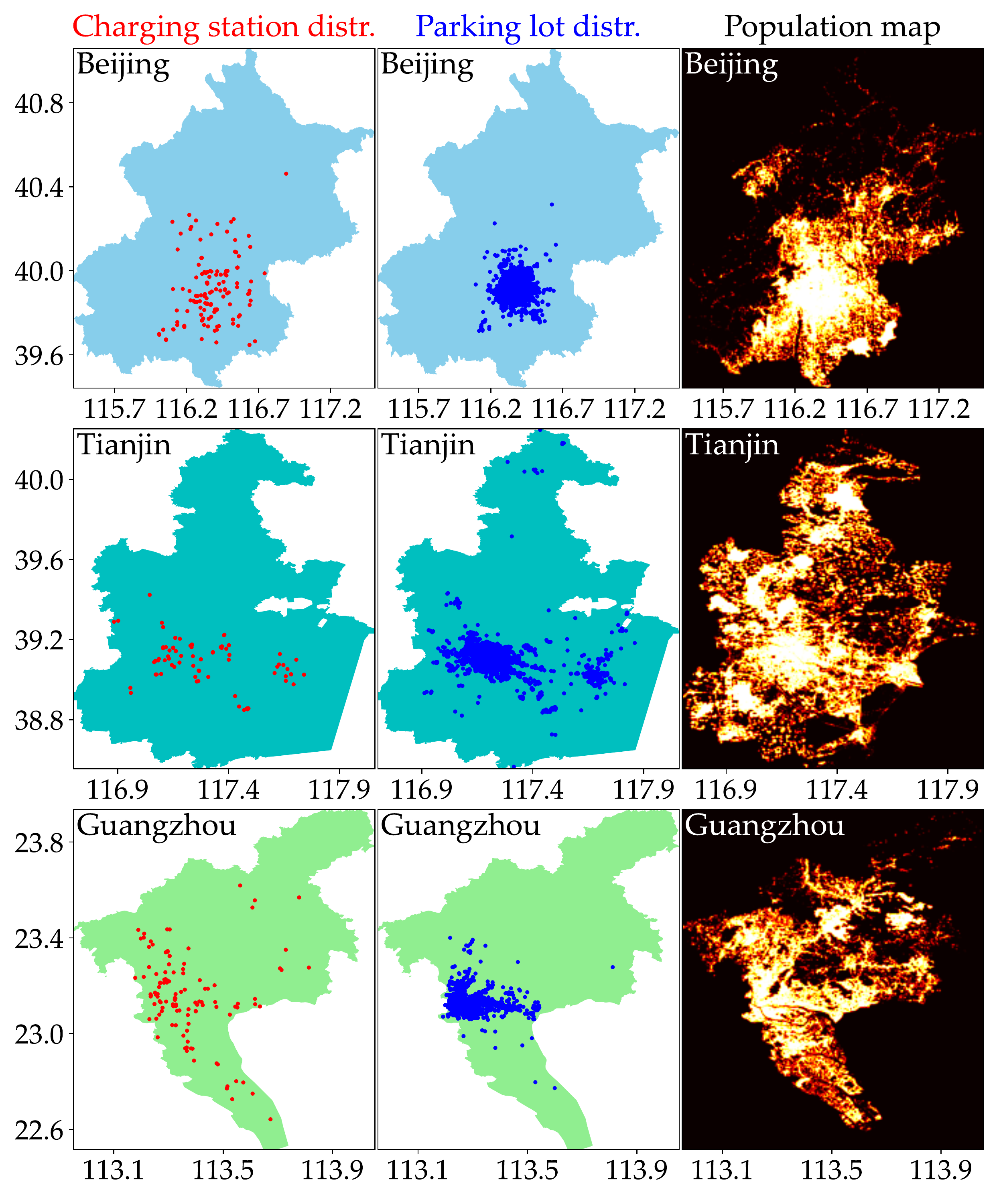}
	\caption{Distributions of charging stations and two proxies for charging demands}
	\label{fig:demand_proxy}
	\vspace{-10px}
\end{figure}

\noindent\textbf{Analysis on Charging Demand Proxies}.
To inspect the representativeness of alternative proxies for charging demands, we compare the distributions of charging stations and two proxies as shown in Fig. \ref{fig:demand_proxy}. We observe that: (1) Parking lots have different spatial distribution with charging stations. In fact, a city has a large number of parking plots belonging to different operators, so we could only obtain parking sessions data in a biased manner (e.g., mainly distributed in the urban centers in Fig. \ref{fig:demand_proxy}). Even if we could collect the comprehensive parking sessions data, they may still have very different spatial-temporal patterns because chargers are not so ubiquitous as parking lots particularly when the EV market share is still small. (2) The population distribution is wider than that of charging stations, which will bring errors to the estimation method where population is allocated to the nearest charging station. Compared with the general population distribution, early EV adopters are disproportionately younger, male, more educated, and more sensitive to environmental concerns \cite{chen2013locating}. In summary, such implicit data have so different distributions with charging demands in nature that they are inappropriate to represent EV charging demands for charger planning.

\noindent\textbf{Performance Comparisons with Varied Budgets}.
From Figs. \ref{fig:evaluation:planning:scalability:budget:b2g}, \ref{fig:evaluation:planning:scalability:budget:b2t} and \ref{fig:evaluation:planning:scalability:budget:g2t}, we observe 1) the revenues achieved by all the algorithms increase with the budget; 2) our \textit{TIO} achieves the highest revenue under all the cases, and its advantage is more obvious as the budget increases, indicating that \textit{TIO} is able to utilize the budget more efficiently on those chargers with higher demands;
3) \textit{CG} performs better than \textit{Even} in Tianjin but the results are just the opposite in Guangzhou when there is a big budget ($>$ \textyen20 million), due to the same reasons as explained before. 
In addition, we want to emphasize that, our \textit{TIO} applies to various city-pair cases, while \textit{CG} is unpractical in a new city due to lack of historical demand data before the actual deployment.

\begin{figure}[t]
	\begin{subfigure}{0.5\linewidth}
	    \captionsetup{width=0.9\linewidth}
		\includegraphics[width=\linewidth]{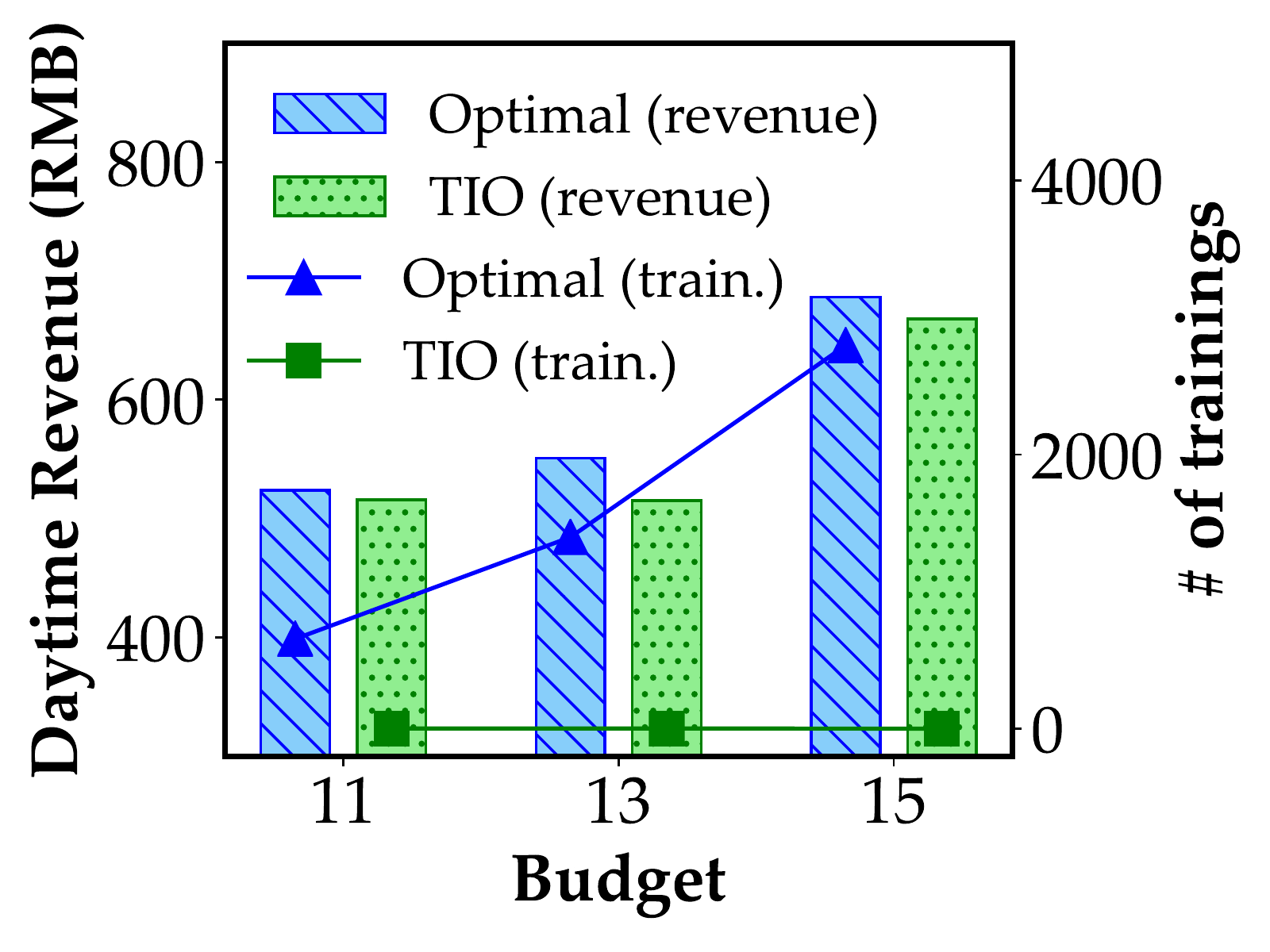}
		\subcaption{Results with varied $B$ ($|C_{TC}|=4)$}
		\label{fig:planning:optimal:budget}
	\end{subfigure}%
	\begin{subfigure}{0.5\linewidth}
	    \captionsetup{width=0.9\linewidth}
		\includegraphics[width=\linewidth]{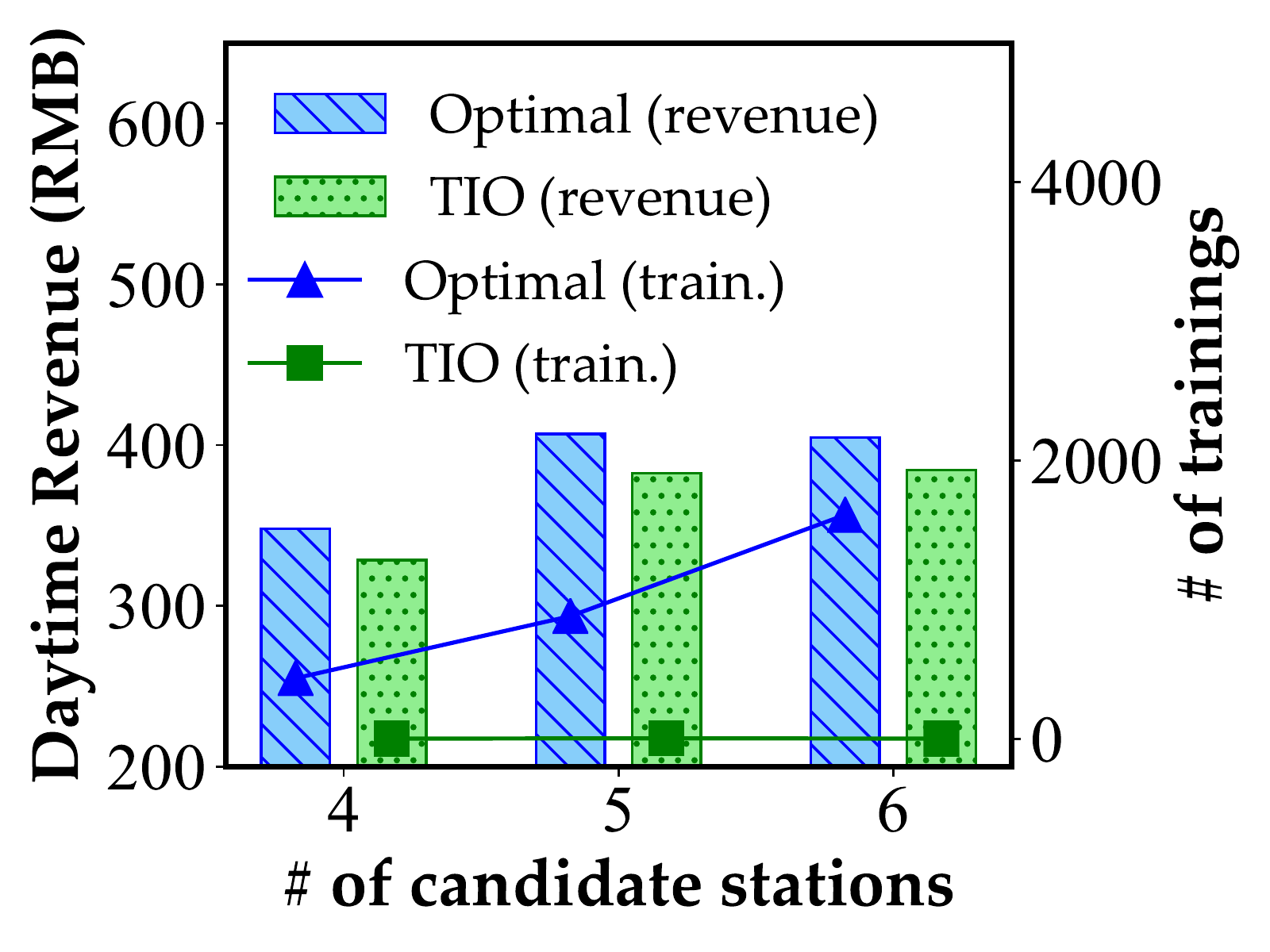}
		\subcaption{Results with varied $|C_{TC}|$ ($B=12$)}
		\label{fig:planning:optimal:candidate_station_number}
	\end{subfigure}
	\caption{Comparison with the optimal solution (BJ$\rightarrow$TJ)}
	\label{fig:evaluation:planning:optimal}
	\vspace{-15px}
\end{figure}

\noindent\textbf{Comparison with the Optimal Solution}. Since \textit{TIO} is a heuristic solution, we are interested to know its effectiveness and efficiency compared with the optimal solution.
Nevertheless, with the large search space analyzed in Sect. \ref{subsec:chllanges}, the optimal algorithm is unpractical. Thus, we select at most 6 candidate stations in the central area of Tianjin and small budgets ($\leq$15, meanwhile we proportionally set $e_i^S=2\text{ and }e_i^F=3$) for experiments.
Fig. \ref{fig:planning:optimal:budget} compares the results with varied $B$ when $|C_{TC}|=4$, and Fig. \ref{fig:planning:optimal:candidate_station_number} compares the results with varied $|C_{TC}|$ when $B=12$.
We observe that the revenue achieved by \textit{TIO} is very close to the optimal solution.
However, the required number of trainings by the optimal solution dramatically increases with $B$ and $|C_{TC}|$, up to 2796 when $B=15$ and $|C_{TC}|=4$.
By contrast, our \textit{TIO} only needs at most 4 trainings.

\noindent\textbf{Time Efficiency.} We evaluate the time cost with the real budgets. The \textit{TIO} at most consumes 2.7 hours with 9 iterations in BJ$\rightarrow$GZ; 2.14 hours with 8 iterations in BJ$\rightarrow$TJ; 3.34 hours with 13 iterations in GZ$\rightarrow$TJ, which is completely acceptable in reality. 

%% file: 6discussion.tex
\section{Discussion}
In spite of many merits for our \textit{SPAP}, some possible limitations are still worthy of discussions or further research in the future, summarized as follows:

\noindent \textbf{Cross-city Prediction}. 
Demand prediction is a challenging task in a new city where no explicit historical data is available.
Although we have designed the \textit{AST-CDAN} model for addressing the domain shift problem, the performance may be degraded specially when source city and target city have quite different characteristics (e.g., city scale, development level and strategy), or source city has low demand diversity.
We plan to tackle this challenge by learning from more source cities to enhance the generalization ability of the transfer learning model.

\noindent \textbf{Cross-city Planning}.
The \textit{TIO} algorithm adopts a heuristic idea without strict guarantee on the optimality.
Nevertheless, it is still promising because 1) the \textit{TIO} algorithm is at least better than any naive method (e.g., ``\textit{even}''), by taking the naive method as the initial plan and iteratively optimizing it;
2) it consistently outperforms various existing charger planning methods; 
and 3) the achieved performance is very close to the optimal solution, which has been verified by extensive experiments.
In the future, other solutions with a solid theoretical guarantee are worth investigating, while this work can provide important insights as a starting point.

\noindent \textbf{Long-term planning}. Given that the EV market is still young, one would need much more data before coming to conclusion on how to construct the whole charging station network.
It could be wise to place chargers in phases, which is also consistent with the gradual development mode commonly adopted by charging station operators in reality.
As one collects data and learns more, the chargers could be placed in other locations in multiple phases or use dynamic pricing as a complement.
Guided by that, this work is committed to solving the cold-start problem in the first phase.
Dynamic urban macro factors, e.g. newly built infrastructure in the future, will influence charging demands of the related regions, 
which should be considered in long-term construction.
Nevertheless, static urban factors used in this work are sufficient for planning in the first phase, whose target is to find a subset of candidate locations 
with the highest utility in the current phase.

%% file: 7relatedwork.tex
\section{Related Work}
\subsection{Charger Demand Modeling and Prediction}
The related work on charger demand modeling and prediction can be classified into into two categories based on the used data type.

\noindent\textbf{Implicit Data.}
A traditional way is to infer charging demands by leveraging relevant \textit{implicit} information \cite{chen2013locating,xiong2017optimal}. 
Chen et al. \cite{chen2013locating} use the parking demand as proxy to estimate the charging demand. 
Xiong et al. \cite{xiong2017optimal} use the population distribution to estimate the charging demand.
Liu et al. \cite{liu2019social} assume that the charging demand is proportional to the traffic flow.
Liu et al. \cite{liu2016optimal} leverage the refueling demand to define the charging demand. 
Unfortunately, such \textit{indirect} method is error-prone due to the dissimilar nature of different spatio-temporal mobility patterns. 
In other words, the implicit data has intrinsic defects for charging demand prediction.

\noindent\textbf{Explicit Data.}
Recently, the advanced data acquisition technologies enable us to collect \textit{explicit} data about charging events of EVs, which helps to charger planning \cite{du2018demand, wang2018bcharge, gopalakrishnan2016demand, li2015growing, luo2020d3p}.
Li et al. \cite{li2015growing} extract charging demands from the seeking sub-trajectories of EV taxis. 
Du et al. \cite{du2018demand} use the return records of an EV sharing platform as the charging demand.
These data sources are only limited to commercial EVs rather than private EVs. 
For the general charging stations except for those that are used exclusively for commercial EVs, the only available \textit{explicit} data are their charger transaction records \cite{gopalakrishnan2016demand}, whereas it is impossible in a new city.

\subsection{Charger Planning}
Existing work on charger planning mainly falls into two categories. In the first category, all charger demands are required to be fulfilled to maximize the social welfare \cite{jia2012optimal, li2015growing, liu2012optimal, xiong2017optimal}. For example, Li et al. \cite{li2015growing} minimize the average seeking and waiting time of all charging demands based on taxi trajectory data. 
The second category takes charging demand as objectives \cite{frade2011optimal, du2018demand, lam2014electric}. 
For example, Du et al. \cite{du2018demand} use both coverage and charging demand as the optimization objective. Our work takes charging demands as part of the objective. However, charging demands are affected by both the station profile and nearby stations, which is ignored by the existing work. Moreover, we are the first to conduct simultaneous demand prediction and planning in a new city.

\subsection{Urban Transfer Learning}

Recently, urban transfer learning \cite{wei2016transfer, katranji2016mobility, wang2018smart, guo2018citytransfer, DBLP:conf/ijcai/WangGMLY19, ding2019learning, liu2018will, he2020human} has emerged to be an effective paradigm for solving urban computing problems \cite{zheng2014urban} by applying transfer learning approaches \cite{pan2009survey}.
Wei et al. \cite{wei2016transfer} tackle the label scarcity and data insufficiency problems.
Katranji et al. \cite{katranji2016mobility} predict the Home-to-Work time for families in a new city using survey data of families in both source and target cities.
Guo et al. \cite{guo2018citytransfer} propose a SVD-based transfer method for chain store site recommendation in a new city. 
Wang et al. \cite{DBLP:conf/ijcai/WangGMLY19} propose a cross-city transfer learning method for deep spatio-temporal prediction tasks.
Ding et al. \cite{ding2019learning} solve the problem of cross-city POI recommendation for the travelers by learning from users' visiting behaviors in both hometown and current city.
However, these works need homogeneous data in the target domain, which is not satisfied in our problem, because there is not any historical charging data in the new city.
On the other hand, the domain generalization technique \cite{muandet2013domain} is leveraged to address the problem of label unavailability in the target domain \cite{liu2018will,he2020human}. 
Liu et al. \cite{liu2018will} detect the parking hotspots of the dockless shared bikes in a new city.
He et al. \cite{he2020human} generate mobility data for a new target city.
However, they have different problem settings with us, as we consider both cross-city demand prediction and station planning simultaneously.

%% file: 8conclusion.tex
\section{Conclusions}
In this paper, we investigate an important problem of planning the charging station network in a new city.
The concept of \textit{simultaneous demand prediction and planning} is first proposed to address the deadlock between charger demand prediction and charger planning.
We prove the NP-hardness of the problem and point out the unacceptable time complexity of a straightforward approach.
We propose the \textit{SPAP} solution by combining discriminative features extracted from multi-source data, an \textit{AST-CDAN} model for knowledge transfer between cities, and a novel \textit{TIO} algorithm for charger planning.
Extensive experiments on real datasets from three cities validate the effectiveness and efficiency of \textit{SPAP}.
Moreover, \textit{SPAP} improves at most 72.5\% revenue compared with the real-world charger deployment.
Our work also has potential implications for other infrastructure planning problems in a new city.